\documentclass[letterpaper]{article}
\usepackage{aaai24}
\usepackage{times} 
\usepackage{helvet} 
\usepackage{courier} 
\usepackage[hyphens]{url} 
\usepackage{graphicx} 
\usepackage{graphicx} 
\usepackage{natbib} 
\usepackage{caption} 
\usepackage{algorithm}

\usepackage{newfloat}
\usepackage{listings}
\DeclareCaptionStyle{ruled}{labelfont=normalfont,labelsep=colon,strut=off} 
\floatstyle{ruled}
\newfloat{listing}{tb}{lst}{}
\floatname{listing}{Listing}
\usepackage{amsmath}

\usepackage{amsthm} 
\usepackage{amssymb}	

\usepackage[capitalize,nameinlink]{cleveref}

\usepackage{slashed}
\usepackage[mathscr]{euscript}
\usepackage[toc,page]{appendix}
\usepackage{dsfont}
\usepackage{tipx}
\usepackage[noend]{algpseudocode}
\usepackage{comment}

\newcommand{\op}[1]{\operatorname{#1}}

\newcommand{\gv}[1]{\ensuremath{\mbox{\boldmath$ #1 $}}} 
 
\newcommand{\grad}[1]{\gv{\nabla} #1} 
\let\baraccent=\= 
\renewcommand{\=}[1]{\stackrel{#1}{=}} 

\newtheorem{thm}{Theorem}[section]

\theoremstyle{definition}
\newtheorem{dfn}{Definition}
\theoremstyle{remark}
\newtheorem{rmk}{Remark}

\usepackage{url,times,wasysym,rotating,tikz}

\newcommand{\mcA}{\mathcal{A}}

\newcommand{\mcC}{\mathcal{C}}

\newcommand{\mcM}{\mathcal{M}}

\newcommand{\mcO}{\mathcal{O}}

\newcommand{\mcS}{\mathcal{S}}

\newcommand{\mbbE}{\mathbb{E}}

\newcommand{\mbbN}{\mathbb{N}}

\newcommand{\mbbR}{\mathbb{R}}

\newcommand{\pre}{\text{pre}}
\newcommand{\env}{\text{env}}

\title{Coagent Networks Revisited}
\author {
    Modjtaba Shokrian Zini\equalcontrib,
    Mohammad Pedramfar\textsuperscript{\rm 1,}\equalcontrib,
  Matthew Riemer\textsuperscript{\rm 2},
  Ahmadreza Moradipari\textsuperscript{\rm 3},
    Miao Liu\textsuperscript{\rm 4}
}
\affiliations {
    \textsuperscript{\rm 1} Purdue University,
    \textsuperscript{\rm 2,4} IBM Research AI,
    \textsuperscript{\rm 3} University of California, Santa Barbara, \\
    modjtaba.shokrianzini@gmail.com, mpedramf@purdue.edu, 
    mdriemer@us.ibm.com,
    ahmadreza\_moradipari@ucsb.com,
    Miao.Liu1@ibm.com
}

\begin{document}

\maketitle

\begin{abstract}
Coagent networks formalize the concept of arbitrary networks of stochastic agents that collaborate to take actions in a reinforcement learning environment. Prominent examples of coagent networks in action include approaches to hierarchical reinforcement learning (HRL), such as those using options, which attempt to address the exploration exploitation trade-off by introducing abstract actions at different levels by sequencing multiple stochastic networks within the HRL agents. We first provide a unifying perspective on the many diverse examples that fall under coagent networks. We do so by formalizing the rules of execution in a coagent network, enabled by the novel and intuitive idea of execution paths in a coagent network. Motivated by parameter sharing in the hierarchical option-critic architecture, we revisit the coagent network theory and achieve a much shorter proof of the policy gradient theorem using our idea of execution paths, without any assumption on how parameters are shared among coagents. We then generalize our setting and proof to include the scenario where coagents act asynchronously. This new perspective and theorem also lead to more mathematically accurate and performant algorithms than those in the existing literature. 
Lastly, by running nonstationary RL experiments, we survey the performance and properties of different generalizations of option-critic models.
\end{abstract}

\section{Introduction}
\textbf{Background and related works. }
Use of stochastic neural networks is commonplace in reinforcement learning (RL) to execute stochastic actions in the environment. However, hierarchical approaches to RL often conceptualize a notion of an abstract action that is not applied in the environment directly, but rather happens within the mind of an agent in order to decompose the problem. \textit{Coagent networks} \cite{thomas2011conjugate} formalize this learning problem where stochastic actions are taken within an agent's mind in the general case. Specifically, hierarchical RL (HRL) models such as Option-Critic (OC) \cite{bacon2017option} and Hierarchical OC (HOC) \cite{riemer2018learning} attempt to leverage abstract actions to learn complicated tasks. This requires networking multiple stochastic policies that learn and act cooperatively to maximize the return from the environment. 
How to learn coagent networks was first studied in \cite{thomas2011policy}, wherein it was shown that the more biologically plausible approach to learning, i.e. REINFORCE \cite{williams1992simple}, is an unbiased estimator of the policy expected return when applied individually on each coagent. In other words, each coagent trained separately with REINFORCE on their own return is equivalent to REINFORCE on the entire network policy. \cite{kostas2020asynchronous} further generalized this to coagents acting (a)synchronously. However, all previous works assumed a disjoint set of parameters describing each coagent. This non-sharing assumption was lifted in the specific context of HOC \cite{riemer2020role}, with the the more advanced Actor-Critic algorithm implementation.

So far, we have discussed prior and related works on the theory of coagent networks. However, since the advent of coagent networks and its options variants, there have been continued efforts on experimenting with these models. The coagent policy gradient theorem (PGT) has been investigated by \cite{gupta2021structural} to train stochastic feedforward neural nets, where it was shown that while such networks struggle to achieve optimal results in supervised settings, they are more performant in non-iid settings compared to backprop. On the other hand, \cite{chung2021map} proposes an alternative to the REINFORCE called the MAP propagation algorithm on the same type of feedforward networks. This reduces the variance of updates, albeit decreasing computational efficiency and increasing bias in their experiments in OpenAI's Gym. The main advantage of coagent networks framework is the flexibility to design the network and the freedom of choosing the rules by which the network coagents operate. This is a powerful tool that, in theory, should allow the user to exploit the relational biases in the task by reflecting them in the network. One such famous design is that of options, which tries to exploit a decomposition of the optimal policy to multiple skills or phases, and it has by far garnered the most attention by the community. They have been evaluated under famous RL tasks, but also multi-tasking, nonstationary, and non-iid datasets. The popular OC has succeeded when applied to Q-learning on Atari \cite{bacon2017option}, or to continuous action spaces \cite{klissarov2017learnings} and asynchronous parallelization \cite{harb2018waiting}. \cite{riemer2018learning,riemer2020role} further improved OC's results under multi-tasking on Atari games by introducing the HOC architecture, and showed superior performance in nonstationary tasks. In addition, properties and ``qualities" of options has been an active area of research. One prominent issue is that of option length, where some become dominant while others become useless. \cite{harb2018waiting} formulates an answer to what a good option is through the notion of deliberation cost, therefore adding a cost to options that deliberate for too long. In addition \cite{khetarpal2020options,chunduru2020attention} develop models of interest and attention to keep the focus of options on certain (less overlapping) areas of state space, thereby encouraging skill decomposition by options.

\textbf{Contributions and novelties. } Now, we discuss a summary of main contributions and novelties of this work, on  theory, algorithms, and experiments. We first revisit the notion of a Coagent Network with the goal of addressing theoretical gaps in its online learning, and clarifying its theoretical foundation and scope of application. In the main theoretical \cref{sec:definition_CN}, we begin by presenting a 
\begin{itemize}
    \item novel definition of a coagent network that formalizes the rule $\Delta$ under which the network coagents cooperate.
\end{itemize}
This rule shows that actions within coagents establish a path of execution, which itself is a new intuitive representation of the network's overall action in the environment. This then facilitates the writing of the policy gradient as the sum of those of the coagents in \cref{CNPGT}. Using the concept of path of execution, we achieve
\begin{itemize}
    \item a much simpler proof than previous works for the general policy gradient theorem (PGT) of coagent networks, in addition to having no assumption on the sharing of parameters.
\end{itemize} 
Last but not least, we show how to extend our framework to the synchronous settings. This includes a generalization of the formalization of the rule $\Delta$, incorporating the possibility of simultaneous executions of coagents, and thus generalizing the concept of a simple path of execution to that of a directed feedforward graph of execution. Then, the asynchronous case is addressed using the so-called \textit{Markov trick}, to transform the asynchronous non-Markov MDP to the synchronous Markov MDP by augmenting the state space. This trick was also used in \cite{kostas2020asynchronous}.  We use this Markov trick to derive our general PGT for asynchronous networks. As such, we 
\begin{itemize}
    \item generalize the PGT to the (a)synchronous coagent networks framework (no assumption on parameter sharing).
\end{itemize}
While our framework is general and allows, for example, for cycles and loops in the execution path, it is important to outline precisely where and when this theorem applies. This is exhibited in \cref{sec:examples}, where in addition to exactly clarifying the theoretical constraints, we cite examples, previously known and new ones, such as Feedforward Options Network, and also non-examples for this purpose. 

Although our work emphasizes more on theory, we show examples of insights and improvements brought by our theory to the online policy learning algorithm and its experimental implementation. We follow the larger focus of prior works on experimenting with options among all coagent networks. We identify multiple areas of improvements in the previous algorithms derived from the PGT for options networks (\cref{sec:algorithm}). First and foremost, we realize that in all previous works, the update takes place on all options at each time step, instead of only when each is called back, which is what the PGT states. Correcting this immediately enables
\begin{itemize}
    \item a faster runtime and learning time on the order of the number of options.
\end{itemize} 
Then, we identify an incorrect update of the termination function that has been used in the literature and argue how it artificially worsens the issue of early option termination.
\begin{itemize}
    \item We address the early option termination issue by our correction of the termination update and proper tuning of the temperature of termination functions, 
\end{itemize} 
thereby avoiding an additional hyperparameter (such as the deliberation cost) in our analysis. Furthermore, 
\begin{itemize}
    \item we argue theoretically how parents of options in the hierarchy are useful as target networks in the Q-learning for the AC algorithm,
\end{itemize} 
(see more in \cref{rmk:11vs1_targetnetwork} and \cref{FON_app}), and we accordingly make stabilizing changes in the algorithm. We experiment extensively to validate our hypotheses on a nonstationary task using the OC/HOC and the new Feedforward Options Network. Overall, these allow us to show that our implementation outperforms the existing literature (\cref{sec:exp_performance,sec:optlengthandtermtemp}). Finally, we extensively investigate and show the improvements, brought by the changes we made, on the long term stability of training (\cref{appsec:more_on_exp}). Lastly, we attach our code as supplementary material for reproducing the experiments with our algorithmic improvements.
      
\section{Preliminaries}
We will study policies acting in a Markov Decision Process (MDP): A tuple $(\mcS, \{\mcA_s\}_{s\in\mcS}, R, P, \gamma)$ where $\mcS$ is the set of states, $\mcA_s$ is the set of actions $a \in \mcA_s$ available at state $s$, $R : (s,a,s') \to r \in \mbbR$ is the reward function for a transition $(s,a,s')$, $P(s'|s,a)$ is the transition function, and $\gamma \in [0,1]$ is the discount factor.

In the problems we consider, each state contains not only information about the environment but also information about the \textit{state of the agent} $\pi$, describing something internal to the agent that is relevant to its execution. Thus, we will use the subscript $\pi$ in $\mcS_\pi$ to denote this augmented set of states, and use $\mcS_{\env}$ to exclusively refer to states of the environment. In both cases, the exact definition will be dependent on the context. For this part, we will simply use $\mcS$. Further, a subset $\mcS_{\op{init}} \subset \mcS$ is defined as the subset of possible initial states. 

At each time-step $t$, the Markov agent/policy $\pi$ acts on state $s_t \in \mcS$ by action $a_t \in \mcA_{s_t}$ with probability $\pi(a_t|s_t)$. The state then changes to $s_{t+1}$ with probability $P(s_{t+1}|s_t,a_t)$ and the reward $r_t:=R(s_t,a_t,s_{t+1})$ is received by the agent. We define 
\begin{itemize}
    \item the expected return of $\pi$ as $J_\pi := \mbbE_{\pi,\mcS_{\op{init}}}[\sum_{t=0}^{\infty} \gamma^t r_t ]$,
    \item the value function of $\pi$ as $V_\pi(s_0) := \mbbE_\pi[\sum_{t=0}^{\infty} \gamma^t r_t | s_0 ]$ which is the expected return given initial state $s_0 \in \mcS_{\op{init}}$,
    \item and the action-value function of $\pi$ as $Q_\pi(s,a) := \mbbE_\pi[\sum_{t=0}^{\infty} \gamma^t r_t | (s_0,a_0)]$ which is the expected reward given initial state-action pair $(s_0,a_0)$. 
\end{itemize}  
On-policy algorithms are one of the main class of algorithms used in online RL to learn the policy with the highest expected return. We can train the parameters $\theta$ describing $\pi$ by ascending the \textit{policy gradient} $\grad_\theta J_\pi$ written as follows:
\begin{align}\label{policygradient}
\grad_\theta J_\pi = \sum_{s_0 \in \mcS_{\op{init}}} d(s_0) \sum_{s\in \mcS} d(s|s_0) \sum_{a \in \mcA_s} \frac{d\pi}{d\theta}(a|s)Q_\pi(s,a).
\end{align}
Here, $d(s_0)$ is the initial distribution for state $s_0 \in \mcS_{\op{init}}$ and $d(s|s_0)$ is the \textit{discounted} state probability of reaching state $s$ from $s_0$, meaning
\begin{align}\label{discountedprobability}
    d(s|s_0)= \sum_{l=1}^{\infty} \sum_{\substack{(s_0,\ldots , s_l) \\ s_l= s}} \sum_{i=0}^{l-1} \gamma^i P_{\pi,\mcS}(s_{i+1}|s_i),
\end{align}
where $P_{\pi,\mcS}(s_{i+1}|s_i)$ is the transition probability of $s_i$ to $s_{i+1}$ given policy $\pi$ and transition function $P$ of $\mcS$. While the equations and theorems in this paper are written for a finite set of states and actions, it is straightforward to generalize them to the infinite case.

\section{Coagent Network}\label{sec:definition_CN}
In this section, we offer a new perspective on coagent networks by formalizing the rule of execution, benefits of which are shown in (1) deriving our main result, the policy gradient theorem, with an intuitive short proof, (2) generalizing it the (a)synchronous setting, and (3) (re)designing previously known and new coagent networks in the next section. We study an agent $\Pi$ composed of a network of Markov policies $\pi_o$, also called coagents, where at each time-step, rules set by the user determine the sequence of policies $(\pi_{o_1},\ldots,\pi_{o_k})$ that act.

\begin{dfn}\label{CNdfn}
Let $\Pi$ be a Markov policy acting within an MDP $\mcM$ with an augmented state space $\mcS_\Pi$ and action space $\mcA_x$ for $x \in \mcS_\Pi$. Let $\mcC_\Pi = \{\pi_o\}_{o \in \mcO_\Pi}$, where $\pi_o$ are Markov policies, indexed with a set of nodes $\mcO_\Pi$, with state and action space $\mcS_o, \mcA_{x_o}$ for $x_o \in \mcS_o$. $\Pi$ is a \textbf{Markov single reward coagent network} with coagent set $\mcC_\Pi$ if
\begin{itemize}
  \item Every state $x \in \mcS_\Pi$ uniquely determines a coagent $\pi_{o_1} \in \mcC_\Pi$ with state $x_{o_1} \in \mcS_{o_1}$, which will be the first coagent in a sequence that will define the action of $\Pi$. 
  \item This sequence is generated using a rule $\Delta$ designed and fixed by the user and not influenced by the parameters, and determines the next acting coagent, along with its input state, using the state $x_o$ and action $u_o$ of the previous coagent:
\begin{align}
    \Delta : (x_o,u_o) \to (o',x_{o'}) , \ \ x_{o'} \in \mcS_{o'}.
\end{align}
$\Delta$ can be stochastic if $o$ is the last coagent that performs a primitive action within a stochastic environment. \textit{Primitive actions} are those that change the state of the environment. In that case, $x_{o'}$ may not be deterministically determined. Moreover, $\Delta$ can also be stochastic in the selection of the next coagent $o'$. For simplicity, we shall assume that $\Delta$ is only stochastic in the former, and there is no randomness in the selection of $o'$.  
  \item Eventually, after finitely many applications of $\Delta$, a coagent applies a primitive action $a$ that leads to the next state $x' \in \mcS_\Pi$. Note this does not mean that the last coagent's action is exactly $a$, as it could contain additional information necessary for the next execution of $\Pi$. 
\end{itemize}
\end{dfn}
\textbf{Execution paths. }$\Pi$'s actions can be viewed as primitive actions, leading to the policy gradient (\ref{policygradient}). Alternatively:
\begin{dfn}\label{execpathdfn}
An action of $\Pi$ is an \textbf{execution path} $P=( (x_{o_i},u_{o_i}) )_{i=1}^k$ where $k$ coagents $\pi_{o_i}$ actions lead to a primitive action $a$, with the last action $u_{o_k}$ containing that information.
\end{dfn}
As a result of this, we can compute $Q_\Pi(x,a)$ as $\mbbE_{P:u_{o_k} \text{leading to } a}[Q_\Pi(x,P)]$.

\textbf{Computing reward for coagents. }Assume $\pi_o$ acts on $x_o^{(0)}$ by $u_o$ and only acts again at $x_o^{(l)}$, meaning after $\Pi$ has gone through $l$ many execution paths. The \textit{transition tuple} for $\pi_o$ is defined as $(x_o^{(0)},u_o,x_o^{(l)})$. Similar to how the cumulative reward for $\Pi$ is computed using the rewards $R(x_t,a_t,x'_t)$ from its transitions, $R$ can also be used to uniquely define the reward functions for all coagents transitions. Let $\Pi$'s states and primitive actions in the time interval associated to $\pi_o$'s transition $(x_o^{(0)},u_o,x_o^{(l)})$ be $(x^{(i)},a^{(i)})_{i=0}^{l-1}$. The reward $R_o$ for the transition $(x_o^{(0)},u_o,x_o^{(l)})$ is 
\begin{align}\label{cumrewardpiodfn}
    R_o(x_o^{(0)},u_o,x_o^{(l)}) := \sum_{i=0}^{l-1} r_i\gamma^i,
\end{align} 
where $r_i = R(x^{(i)},a^{(i)},x^{(i+1)})$. Notice that if $\pi_o$ is never reactivated (i.e. $l=\infty$), then the reward corresponding to $(x_o^{(0)},u_o)$ is an infinite sum.
\begin{rmk}
The reward \textit{source} of all coagents is the same $R$, hence the name \textit{single} reward in our definition. The motivation and necessity of this assumption is explained later in \cref{rmk:singlereward}.
\end{rmk}
\textbf{Convention. }Throughout this text, we will use the term Coagent Network (CN) to refer to a Markov single reward coagent network.

Coagent networks have a natural graph representation.

\begin{dfn}\label{graph_rep}
The graph of a CN $\Pi$ has vertex set $\mcO_\Pi$ and directed edges which shows which coagent could follow the next in some execution path.
\end{dfn}
\begin{rmk}
A loop in a graph simply means that the coagent can act multiple times in a row. Furthermore, one can start with the most refined decomposition of $\Pi$ into stochastic policies and contract any edge to derive different graph representations. The policy gradient will obviously not change as $\Pi$ has not changed; indeed, as shown in our main theorem, the policy gradient for the new composite coagent decomposes to the sum of the two original ones. For an example of a coagent network graph, see \cref{treehocfigure1}.
\end{rmk}

\subsection{Policy Gradient Theorem}
Let $\nabla_\theta J_\Pi$ be the policy gradient for $\Pi$ with parameters shared among coagents, written as:
\begin{align}\label{policygradientofPi}
    \grad_\theta J_\Pi= &\sum_{x_0 \in \mcS_{\Pi,\op{init}}} d(x_0) \sum_{x \in \mcS_\Pi} d(x|x_0) \nonumber \\
    &\sum_{P \in \mcA_x}\frac{d\Pi}{d\theta}(P|x) Q_\Pi(x,P),
\end{align}
where actions are viewed as execution paths $P=( (x_{o_i},u_{o_i}) )_{i=1}^k$ with $u_{o_k}$ implying a primitive action. We may drop $\theta$ from $\nabla_\theta$ as it is implied from the context.
\begin{thm}\label{CNPGT} The policy gradient is the sum of the policy gradient of the coagents:
\begin{align}\label{mastereq}
    \grad J_\Pi = \sum_{x_0 \in \mcS_{\Pi,\op{init}}} d(x_0) \sum_{o} \sum_{x\in S_{\Pi}, x_o \in \mcS_{\pi_o}} d(x_o,x|x_0) \nonumber \\
    \sum_{u_o \in \mcA_{x_o}}\frac{d\pi_o}{d\theta}(u_o|x_o)Q_{\pi_{o}}(x_{o},u_{o}).
\end{align}
Here, $d(x_o,x|x_0)$ is the discounted probability of reaching state $x$ from $x_0$, i.e. $d(x|x_0)$, multiplied by the probability of reaching $x_o$ from $x$ within a \textbf{single} execution path.
\end{thm}
\begin{proof}
Note $\Pi(P|x) = \prod_{o \in P}\pi_o(u_o|x_o)$ where $o \in P$ means every coagent $o_i, 1\le i \le k$. Rewriting (\ref{policygradientofPi}):
\begin{align}\label{eq:rewrite_policygradientofPi}
    &\sum_{x_0 \in \mcS_{\Pi,\op{init}}} d(x_0) \sum_{x \in \mcS_\Pi} d(x|x_0) \nonumber \\
    &\sum_{P \in \mcA_x}\frac{d}{d\theta}\Big( \prod_{o\in P} \pi_o(u_o|x_o) \Big )Q_\Pi(x,P).
\end{align}
Distributing the derivative over the product leads to
\begin{align}\label{dist_deriv}
    &\sum_{o} \Big( \prod_{o'\in P_{<o}} \pi_{o'}(u_{o'}|x_{o'}) \Big )\frac{d\pi_{o}}{d\theta}(u_{o}|x_{o}) \cdot \nonumber \\
    &\Big( \prod_{o'\in P_{>o}} \pi_{o'}(u_{o'}|x_{o'}) \Big ),
\end{align}
where $P_{<o},P_{>o}$ are the parts of $P$ before and after $o$. The first product over $P_{<o}$, summed over all possible $P_{<o}$ paths leading to $x_o$, gets absorbed by $d(x|x_0)$ and gives us $d(x,x_{o}|x_0)$ by definition. For the product over $P_{>o}$, summed over all possible $P_{>o}$ paths with state-action $(x_o,u_o)$ for $o$, gets absorbed by $Q_\Pi(x,P)$ and gives $Q_{\pi_{o}}(x_{o},u_{o})$. This is proved by iteratively using a Bellman-type equation for $\pi_o$ in terms of the next coagent in the execution path:
\begin{align}\label{bellmanCN}
    Q_{\pi_o}(x_o,u_o) = \sum_{u_{o'}} \pi_{o'}(u_{o'}|x_{o'})  Q_{\pi_{o'}}(x_{o'},u_{o'}),
\end{align}
where $o'$ is the next coagent according to $\Delta$. Applied iteratively, this gives $Q_{\pi_o}(x_o,u_o) =$
\begin{align}\label{bellmanCNfinal}
     \sum_{P_{>o}} \prod_{o' \in P_{>o}} \pi_{o'}(u_{o'}|x_{o'}) Q_{\pi_{l_{P_{>o}}}}(x_{l_{P_{>o}}},u_{l_{P_{>o}}}),
\end{align}
where the summation is over all paths $P$ that include $(x_o,u_o)$, and $l_{P_{>o}}$ is the index of the last coagent in $P_{>o}$, which action immediately leads to a primitive action, thus $Q_{\pi_{l_{P_{>o}}}}(x_{l_{P_{>o}}},u_{l_{P_{>o}}}) = Q_\Pi(x,P)$.
\end{proof}
\begin{rmk}
Our much shorter general proof and definition compared to those of previous works (which also had more assumptions on the model) show the benefit of formalizing the rule $\Delta$. In comparison, this avoids many of the calculations around the discounted probability distributions identities \cite{riemer2018learning,riemer2020role}, as one can use the graphical representation of the network to reason through the algebra by simply following the relevant paths of execution as demonstrated in the proof above.
\end{rmk}
\begin{rmk}\label{rmk:singlereward}
Note that our proof also uses the single reward source assumption, that $R$ determines all rewards $R_o$. Otherwise, if $R_o$ were to include any `pseudo' or `intrinsic' reward, as in the goal-based HAC model \cite{levy2017hierarchical}, relating $Q_\Pi$ to $Q_{\pi_o}$ as in (\ref{bellmanCN},\ref{bellmanCNfinal}) would not have been possible (see \cref{nonexamples_cn} for further details).
\end{rmk}
\begin{rmk}
The application of this theorem to any specific CN consists of (1) identifying the coagents and $\mcS_\Pi,\mcS_o$, and (2) computing the discounted state action probability $d(x_o,x|x_0)$ (see e.g. \cref{HOCPolicyGradientTheorem} for HOC).
\end{rmk}
\subsection{(A)Synchronous Coagent Networks}\label{appssec:synchronous}
Our definition of coagent networks needs to be slightly generalized to include the (a)synchronous setting \cite{kostas2020asynchronous}. In brief, the rule $\Delta$ needs to accommodate simultaneous executions for synchronous CNs. Then, the asynchronous setting is recovered by augmenting the state space. 

\textbf{Synchronous Coagent Network. }The rule $\Delta$ of a CN dictates the path of execution. In particular, our definition of a path of execution is a sequence of coagents states and actions. In the synchronous framework, the policy $\Pi$ is instead composed of a fixed feedforward network with $N$ layers of coagents where all coagents in the same layer execute simultaneously. The $i-$th layer's states $X_i = (x_{i,o})_o$ (with $o$ going over all coagents in layer $i$) is given by its previous layer outputs into $o$, called $U_{i,o}^\pre$ and the environment state $s$, i.e. $x_{i,o} = (U_{i,o}^\pre,s)$. Hence, a path of execution would be a sequence of \textit{sets} of coagents where coagents within the same set execute simultaneously, and their collective actions and states determine those of the next set. Therefore, the only change to our definition is in the rule $\Delta$, to become a function of the form $\Delta : \{(x_{i,o},u_{i,o})\}_{o} \to \{(x_{i+1,o'})\}_{o'}$.

To adapt our theorem, we take the following approach. First, formally consider all coagents in the same layer as being one coagent. This means we have a coagent network which is simply one line of coagents $\Pi_1, \ldots ,\Pi_N $ acting one after the other. Here, we are in our original framework and our theorem for this sequence of coagents gives
\begin{align}
    \nabla J_\Pi = &\sum_{x_0 \in \mcS_{\Pi,\op{init}}} d(x_0) \sum_{i=1}^N \sum_{x\in S_{\Pi}, X_i \in \mcS_{\Pi_i}} d(X_i,x|x_0) \nonumber \\
    &\sum_{U_i \in \mcA_{X_i}}\frac{d\Pi_i}{d\theta}(U_i|X_i)Q_{\Pi_{i}}(X_{i},U_{i})
\end{align}
where capital letters are used for the states and actions of $\Pi_i$, which is composed of those of its constituting coagents $\{\pi_{i,o}\}_o$, i.e. $X_i = (x_{i,o})_o, U_i = (u_{i,o})_o$. Now our problem is to rewrite the policy gradient of coagent networks such as $\Pi_i$, as the sum of its simultaneously acting coagents. This requires an analysis similar to \cref{dist_deriv}, by distributing $\frac{d}{d\theta}$ over the coagents acting simultaneously
\begin{align}\label{eq:synch_before_marginalization}
\begin{split}
    &\sum_{x_0 \in \mcS_{\Pi,\op{init}}} d(x_0) \sum_{i=1}^N \sum_{x\in S_{\Pi}, X_i \in \mcS_{\Pi_i}} d(X_i,x|x_0) \sum_{U_i \in \mcA_{X_i}} \\
    &\sum_{o} \Big( \prod_{o'\neq o} \pi_{i,o'}(u_{i,o'}|x_{i,o'}) \Big )\frac{d\pi_{i,o}}{d\theta}(u_{i,o}|x_{i,o})Q_{\Pi_{i}}(X_{i},U_{i})
\end{split}
\end{align}
Similar to our main theorem, we follow this by taking appropriate marginalization over undifferentiated coagents, state-action occupancies, and $Q$-values (carried out in more details in \cref{appssec:synchronous_full_proof}). They yield the desired decomposition to gradients of the coagent policies:
\begin{align}
\begin{split}
        \nabla J_\Pi = &\sum_{x_0 \in \mcS_{\Pi,\op{init}}} d(x_0) \sum_{i,o} \sum_{x\in S_{\Pi}, x_{i,o} \in \mcS_{\pi_{i,o}}} d(x_{i,o}, x|x_0)  \\
        &\sum_{u_{i,o} \in \mcA_{x_{i,o}}}  \frac{d\pi_{i,o}}{d\theta}(u_{i,o}|x_{i,o})Q_{\pi_{i,o}}(x_{i,o},u_{i,o}). \nonumber
\end{split}
\end{align}

\textbf{Asynchronous Coagent Network. }Defined in \cite{kostas2020asynchronous}, an asynchronous coagent network is one where the coagents may act asynchronously. For example, one node at environment time step $t$ may have waited to act based on the actions taken by others in the step $t-10$. One can see the asynchronous setting as a non-Markov version of the synchronous coagent network. More generally, using the Markov trick, i.e. changing the non-Markov MDP to a Markov MDP by augmenting the state space by adding histories to the states, we have:
\begin{thm}\label{gen_result}
The policy gradient of $\Pi$ for a single reward coagent network that is formed by non-Markov coagents can be written as the sum of the policy gradient of its coagents.
\end{thm}
To prove the above for the particular example of asynchronous networks, we first apply the Markov trick, which in this case, is exactly the same state augmentation trick that \cite{kostas2020asynchronous} employed to translate this problem into a synchronous MDP problem. Then one can write the shared parameter version of the synchronous setting as previously proved. Next, one must show that the synchronous coagents of this new MDP have the same policy gradients as the original ones. Again, this equivalence is achieved by following the same reasoning in \cite[Sec. 5.1, p. 7]{kostas2020asynchronous} which, although done for the non-shared parameter version, is irrelevant of the parameters, and only revolves around the definition of the state augmentation and the two MDPs equivalence (\cref{appssec:asynchronous_full_proof}).

\section{Examples and Applications}\label{sec:examples}
We discuss herein how our definition easily incorporates previous models and inspires new ones.

\textbf{1) Hierarchical Option Critic (HOC):}
For illustrating the expressivity of our definition, let us model an HOC \cite{riemer2020role} as a CN. An HOC has $N$ levels of hierarchy with $\pi^1$, sitting at the first level of hierarchy, as the most abstract option and selecting the first option $o^1$ with probability $\pi^1(o^1|s)$ given a state $s$. 

We view HOC as a \textit{tree of options}, denoted by $\left<m_1,\ldots,m_N\right>$ where parents at layer $i$ have $m_{i+1}$ many children.  An option is equipped with a policy $\pi$ which selects a child and a termination function $\beta$. Each option is uniquely determined by a tuple $o^{1:j}=(o^1,\ldots,o^j)$ where $o^i \in \{1,\ldots,m_i\}$, and $0\le j\le N-1$ with the most abstract option (the root) labelled by $o^{1:0}:=o^0$. The option $o^{1:j}$, sitting at the $j+1-$st level of hierarchy, has a policy $\pi_{o^{1:j}}(\cdot|s)$, also denoted by $\pi^{j+1}(\cdot|s,o^{1:j})$ which selects a child $o^{1:j+1}$, and termination function $\beta_{o^{1:j}}(\cdot)$, also denoted by$\beta_{j}(\cdot,o^{1:j})$, which decides to terminate with probability $\beta_{j}((s,o^{1:j}),1)$ and to not terminate with probability $\beta_{j}((s,o^{1:j}),0) := 1-\beta_{j}((s,o^{1:j}),1)$. We will abuse notation and denote $\beta_{j}((s,o^{1:j}),1)$ by $\beta_{j}(s,o^{1:j})$. We may also use `node' while referring to an option in this tree. 

Denote the HOC policy by $\Pi$. The states $\mcS_\Pi$ are tuples of the form $(s,o^{1:i},d)$ or $(s,o^{1:i},u)$. We may simplify the notation by dropping ``$1:$'' from the subscript index and use $s_{o,i,d}$ or $s_{o,i,u}$. $s_{o,i,d}$ ($s_{o,i,u}$) means the state of the environment is $s$, and the algorithm is about to execute $\pi_{o^{1:i}}$ ($\beta_{o^{1:i}}$), as it is in the $d= \ $downward ($u =\ $upward) execution mode. 

We describe the rule $\Delta$ for HOC. Starting from the root option policy $\pi^1$ at $s_{o,0,d}$, each node takes and action, moving from $s_{o,j,d} \to s_{o,j+1,d}$. Eventually, a node $o^{1:N-1}$ selects a primitive action $a=o^N$ changing the state of the environment to $s'$. Then the upward mode is initiated, i.e. $s'_{o,N-1,u}$. In this mode, termination functions get selected along the previous path and they decide whether to terminate or not. If a node terminates, the parent's termination function is called. Therefore the state moves from $s'_{o,j+1,u} \to s'_{o,j,u}$ with probability $\beta_{j+1}(s_{o,j+1})$. At some level $l+1$, node $o^{1:l}$ does not terminate, with probability $1-\beta_{l}(s_{o,l})$ and the mode changes back to $d$, i.e. $s'_{o,l,u} \to s'_{o,l,d}$. Then the downward execution applies as before until a new primitive action is selected by some $o'^{1:N-1}$. Given this rule $\Delta$, we derive the associated graph representation as shown in \cref{treehocfigure1}. 
\begin{figure}[h]
    \centering
    \includegraphics[width=\columnwidth]{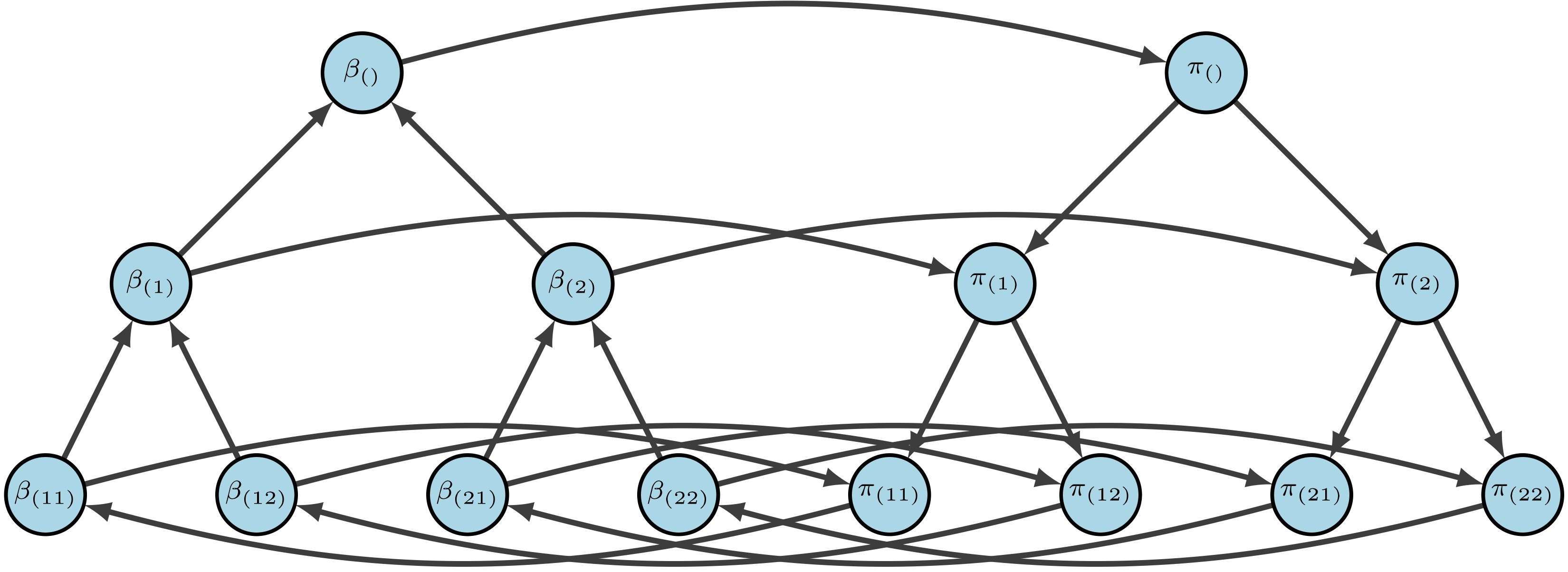}
    \caption{A graph representing the HOC $\langle 1,2,2\rangle$. The policy and termination function of an option each form a coagent. Nodes are labelled by their unique address and type (policy or termination). The termination functions send an edge to their corresponding policies. The policies of the leaf nodes are also followed by their corresponding termination after applying a primitive action.}
    \label{treehocfigure1}
\end{figure}

\textbf{2) Feedforward Option Network (FON):}\label{FON}
Our CN framework allows us to consider other examples such as the simple feedforward network where options at level $i$ are connected to all the next ones at level $i+1$. This CN is used in our experiments as a toy model to better understand the option properties and how that is affected if parents share children. The rule $\Delta$ is similar to that of HOC. We denote by $\left< m_1,m_2,\ldots,m_{N} \right>$ an $N$ layer FON where the $i$-th layer has $m_i$ many options. In all cases, there is a single root, thus $m_1=1$. In our experiments, we tested configurations such as $m_i =1, \forall i$, and $\left < 1,2,\ldots, 2 \right >$. 

\textbf{3 and 4) Options of Interest \cite{khetarpal2020options} and Stochastic Neural Nets \cite{gupta2021structural,chung2021map}:} There are more examples of coagent networks which we discuss further in \cref{appsec:more_on_examples}.

\textbf{Non-examples of CNs:} 
Any model that violates the stochasticity of the coagents or the single reward source assumption could be considered as not directly fitting into our definition and would require examination. Two important models of this nature are the Attention Option-Critic \cite{chunduru2020attention}, violating the stochasticity of the coagents, and Hierarchical Actor-Critic (HAC) \cite{levy2017hierarchical}, violating the single reward source assumption. However, both can in theory be incorporated in our framework if one adds stochasticity, or uses the Markov trick (see \cref{nonexamples_cn}). 


\section{Practical Insights From Our Theory}\label{sec:algorithm}

\cref{TabularAlg} shows how to leverage \cref{CNPGT} and apply the Actor-Critic (AC) algorithm to learn the options in an HOC/FON network. Below are our improvements and comparisons derived from the theory.

\textbf{1) Update only on arrival and its runtime advantage. }
Previous works \cite{riemer2018learning,riemer2020role,khetarpal2020options,bacon2017option,chunduru2020attention} update every parent of the current active option at each time-step. However, \cref{CNPGT} only lists the gradients $\frac{d\pi_o}{d\theta}$ at states $x_o$, where $\pi_o$ is in the execution path. Thus the more accurate way to update options is to do so only when they are called back. Otherwise, depending on the levels of hierarchy $N$ and average duration of options, the runtime of each time step is increased by a factor of $O(N)$. As demonstrated in the experiments (\cref{sec:exp_performance}), updating each time step does not lead to less episodes for learning either.
Thus, the learning time is also increased by a factor of $O(N)$.
\begin{algorithm}[h!]
\caption{HOC/FON Actor-Critic}\label{TabularAlg}
\begin{algorithmic}[1]
\Procedure{AC}{$\textit{env},t_{\max},\left <m_i\right>_{i=1}^N$, $\alpha_Q,\alpha_\pi,\alpha_\beta,\gamma,\theta$}
\State $o \gets \textit{RootOption}$
\State $\textit{PathToRoot} = [o]$
\State $t=0$
\State $s \gets s_0$  // set initial state
\State // select options for initial step
\State $\omega, \textit{PathToRoot} = \textit{PrimitiveAction}(s, t, o)$
\Repeat 
\State // take an action and step in the environment
\State $s^{'},r, \textit{done} \gets \textit{env}.\textit{step}(\omega)$
\State // Computing rewards $R_o$ (see (\ref{cumrewardpiodfn}))
\For{$o \in \textit{PathToRoot}$}
\State $o.Reward \gets o.Reward + \gamma^{t-o.ActivationTime-1}r$
\EndFor
\For{$o \in \textit{PathToRoot}$}
\If{$o.\textit{SampleTermination(s)}$}  
\State $o.\textit{Terminated} \gets \textit{True}$
\Else
\State $\omega \gets o$
\State \textbf{break}
\EndIf
\State // $\omega$ is the option that did not terminate
\EndFor
\State $\text{UpdateCritics}(s, \textit{PathToRoot}, t, \textit{done})$
\State $\text{UpdateActors}(\textit{PathToRoot}, t, \textit{done})$
\State $\text{UpdateBetas}(s, \textit{PathToRoot},t$, $\omega)$
\For{$o \in \textit{PathToRoot}$}
\If{$o.\textit{Terminated}$}
\State $o.\textit{Terminated} \gets \textit{False}$  
\State // Re-initializing for the next step
\EndIf
\EndFor
\State $s \gets s'$ , $t\gets t+1$
\State $\omega, \textit{PathToRoot}=\textit{PrimitiveAction}(s, t, \omega)$
\Until{\textit{done} or $t > t_{\max}$}
\EndProcedure
\end{algorithmic}
\end{algorithm}

\textbf{2) Incorrect $\beta_o$ update and early termination. }\label{wrongbetaupdate}
We address a recurring mistake in the literature's AC algorithm implementation for updating $\beta$, which worsens one of the issues of OC and HOCs called early option termination. As \cref{CNPGT} states, for the update of an option $o$, the term $\frac{d\beta_o}{d\theta}(s,a)$ should be scaled by the $Q_{\beta_o}(s,a)$. If $a= \ 
$termination, then $Q_{\beta_o}(s,\text{termination}) = V_{\beta_{\text{parent}(o)}}(s)$, i.e. the value of the higher option termination function at $s$. Otherwise, $Q_{\beta_o}(s,\text{not termination}) = V_{\pi_o}(s)$. For the sake of illustration, let us consider the case of the Option-Critic model with $root$ being the parent of $o$. Then $V_{\beta_{\text{parent}(o)}}(s) = V_{\pi_{root}}(s)$, as the root never terminates. Throughout the literature, instead of multiplying the gradient by either $V_{\pi_{root}}(s)$ or $V_{\pi_o}(s)$, the difference of these two terms denoted by $A_{\beta_o}(s) = V_{\pi_o}(s) - V_{\pi_{root}}(s)$ is considered. Hence, independent of the action, the update is always
$$d\theta_{\beta_o} \gets d\theta_{\beta_o}  - \alpha_{\beta_o} \frac{d\beta_o(s)}{d\theta_{\beta_o}} A_{\beta_o}(s),$$
where $\theta_{\beta_o}$ are ${\beta_o}$'s parameters. In theory, it is clear that $A_{\beta_o}(s)$ is always negative as $V_{\pi_o}(s) \le V_{\pi_{root}}(s)$, therefore changing the parameters in the direction of the termination gradient, thus always encouraging early termination. This artificially created problem then becomes the reason behind adding another hyperparameter $\eta>0$ to mitigate this issue \cite{bacon2017option,harb2018waiting}.
Correcting this mistake, while eliminating the need for an additional hyperparameter $\eta$, does not resolve the early termination issue entirely. Indeed, at the beginning of training, since the lower options get their values updated more frequently, as long as rewards are nonnegative and values are initialized similarly, one can expect the value functions of the lower options to be generally higher than those of their parents. This would push the options to terminate less. However, the rate at which the parents are updated is important and this rate may need to be quite small. As supported by experiments (\cref{sec:optlengthandtermtemp}), this can be facilitated by lowering the temperature of termination functions.

\textbf{3) Using the parent as target network. }\label{parentfortargetcomputation}Compared to previous works, another important modification to the algorithm is the use of the parent of the child for the critics target computation, instead of the child itself.
In low temperatures, one needs to fully leverage the parent option as a value target network in the critic update of the child. For an option policy $\pi_o$, the usual $Q$-learning target for $Q_{\pi_o}(s,a)$ has the following form:
$$\delta_o \gets Q_{\pi_o}(s,a) - \big(r + \gamma ^{t} ((1-\beta_o(s^{(t)}))\max_aQ_{\pi_o}(s^{(t)},a) +$$
$$\beta_o(s^{(t)}) (\text{higher option terms}) )\big),$$
where $t$ is the number of time-steps until $o$ is called back, and $r$ the discounted cumulative rewards accumulated during that time. Action $a$ can be a primitive action or a child of $o$. Notice that the multiplier of $\gamma^t$ is simply $V_{\beta_o}(s^{(t)})$ which is decomposes in two terms by considering the two terminating and not terminating cases. In theory, $\max_a Q_{\pi_o}(s^{(t)},a) = V_{\pi_o}(s^{(t)}) = Q_{\pi_{\text{parent}(o)}}(s^{(t)},o)$ and so in practice, the parent can serve the role of a value target network. Therefore, one can make the following change in the update
$$\delta_o \gets Q_{\pi_o}(s,a) - \Big(r + \gamma^{t} 
\big((1-\beta_o(s^{(t)}))\cdot$$
$$Q_{\pi_{\text{parent}(o)}}(s^{(t)},o) + \beta_o(s^{(t)}) (\text{higher parent terms}) \big)\Big).$$
Going further, we implement a similar change to ``$(\text{higher parent terms})$'' where any $\max_a Q_{\pi_{o'}}(s^{(t)},a)$ is replaced by $Q_{\pi_{\text{parent}(o')}}(s^{(t)},o')$ for any $o'$ in the path to root for $o$ (line \ref{UpdateCriticsParentTarget} of \cref{UpdateAndHelperFunctionsTabular}). As supported by our experiments, we hypothesize that this ensures that $Q_{\pi_o}$ for the lowest options $o$ do not get too much larger than those of their parents, as the target includes the action value of the parent instead of the option itself. Hence termination probability drops in a less dramatic way, allowing the higher options to learn enough to be able to play their role as a value target and stabilize the learning. Eventually, this allows us to obtain models with longer lasting options (\cref{sec:optlengthandtermtemp,appsec:more_on_exp}). A similar improvement for $Q_{\beta_o}$ is discussed in \cref{appsec:more_on_exp}.
\begin{rmk}\label{rmk:11vs1_targetnetwork}
Viewing the parent as a target network is also motivated by studying more deeply the crucial difference between a simple FON model $\left < 1,1 \right >$ vs $\left< 1\right> $. As theoretically justified in \cref{FON_app} and experimentally confirmed, $\left < 1,1 \right >$ turns out to be an AC model which is soft-updated using a value target network (the parent) with rollout at a learnable rate (determined by the termination function of the child). This also explains why in other hierarchical networks, such as HAC, it is unnecessary to use a target network. On the other hand, $\left< 1\right> $ is exactly a vanilla AC model. 
\end{rmk} 

\section{Experiments}
We conduct a number of experiments on a nonstationary stochastic sparse variant of the Four Rooms (FR) task \cite{riemer2018learning,riemer2020role}. We invite the reader to \cref{appsec:experiments} and the ones after, which, as referred to, confirm many of the discussed insights and include
\begin{itemize}
    \item A comparison of the performance of OC/HOC/FON to each other and previous works implementations (\cref{treesvsmatt}).
    \item Experiments on option length and how to main performance while tuning the temperature hyperparameter of $\beta_o$ to increase usage of all options.
    \item Long runs to check the stability of our algorithm, confirming the hypothesis of wider networks enjoying more stability in the long runs.
\end{itemize}
\begin{figure}[h!]
    \centering
    \includegraphics[width = \columnwidth]{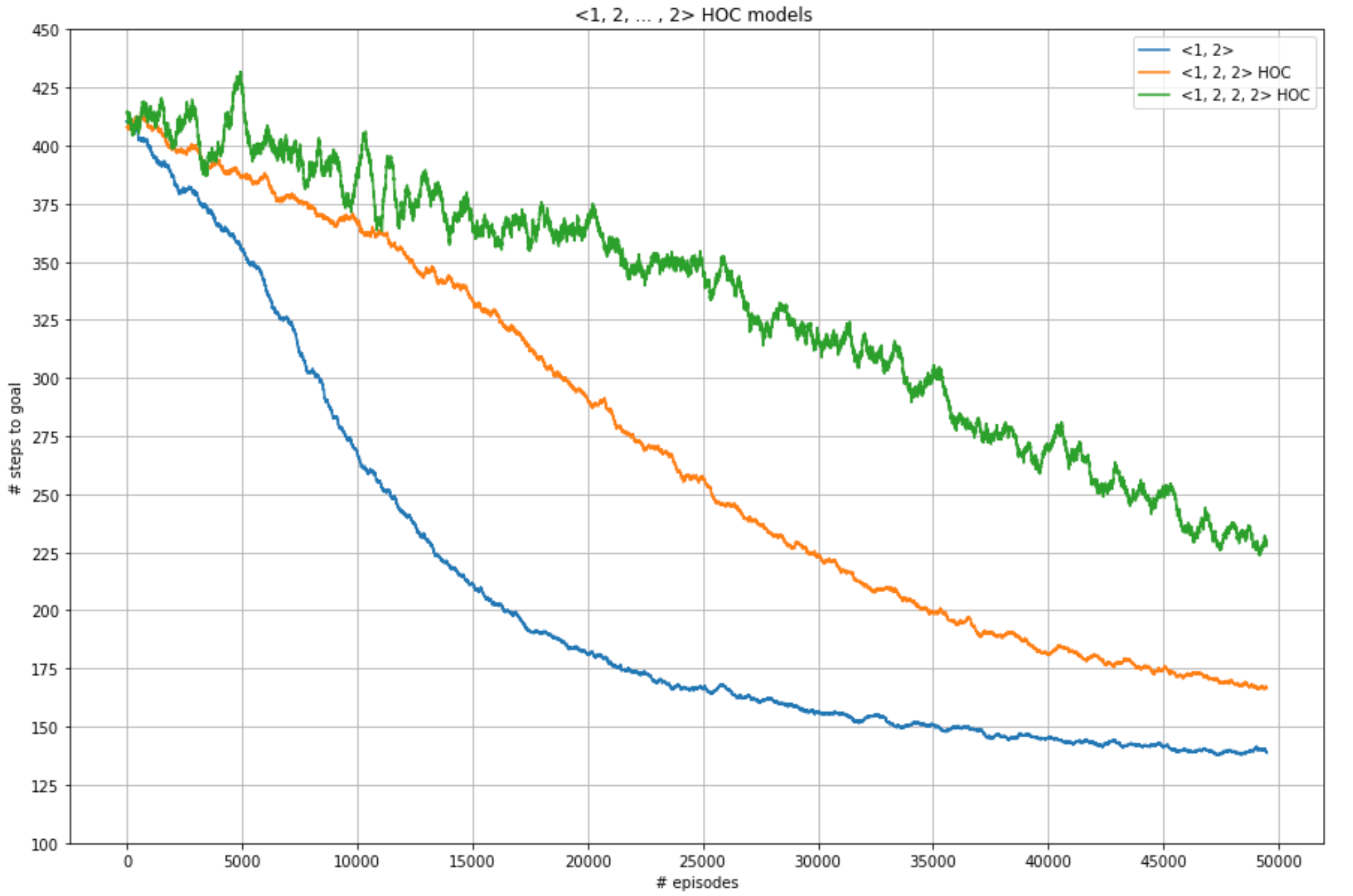}
    \caption{
    Our results for the OC model $\left<1,2\right>$ and HOC models $\left<1,2,2\right>,\left<1,2,2,2\right>$ for 5 random seeds. To be compared with the results from \cite[Fig. 3]{riemer2018learning} for the OC model $\left<1,4\right>$ and HOC models $\left<1,2,2\right>,\left<1,2,2,2\right>$, where our models outperform their counterparts. We also note how the simplest $\left<1,2\right>$ OC model outperforms all other models.
    }
    \label{treesvsmatt}
\end{figure}

\section{Conclusion and Future Works}
In this work, we introduce a new intuitive definition of coagent networks, allowing us to design new ones, and illustrate the idea of execution paths to prove policy gradient theorems even when coagents share parameters or act asynchronously. We survey the performance and option properties of well-known coagent networks while making the algorithms more performant, and provide theoretical analysis supporting the improvements. 
This work attempts to put coagent networks and their algorithms on solid mathematical footing, and future works should focus more on the network design of these models and the suitable tasks that they can be evaluated against. For example, one of the design goals of the options framework is for different options to represent different skills. If so, the training and tasks should accommodate such outcome; having options which focus, or have their attention on different parts of a screen or different parts or phases of the task \cite{khetarpal2020options,chunduru2020attention}, may not necessarily accommodate this. In addition, further algorithmic improvements should be done to mitigate the issues inherent in training such networks, chief among them being the variance \cite{gupta2021structural}, in order to make them competitive against state-of-the-art online RL paradigms.

\section*{Acknowledgements}
We would like to thank Vaneet Aggarwal for his suggestions and comments. The first named author would like to acknowledge the support of the Perimeter Institute for Theoretical Physics and Microsoft during his time of working on this research. Research at Perimeter Institute is supported by the Government of Canada through Innovation, Science and Economic Development Canada and by the Province of Ontario through the Ministry of Research, Innovation and Science. The experiments were conducted using Microsoft computational resources.

\bibliography{main}

\appendix

\section{Hierarchical Option Critic}\label{HOCPolicyGradientTheorem}
In this appendix, we review the graph representation of HOC, and by presenting some useful Bellman equations for the PGT, show how our \cref{CNPGT} for HOC is the same PGT as in \cite[Theorem 1]{riemer2020role}. Note the list of complicated equations in the next parts are simply derived as part of confirming that our result matches with the literature, and should not be seen as proving \cref{CNPGT}, which we accomplished already in the short proof in the main text.
\textbf{Alternative graph representation of HOC. } It is possible to define the CN matching HOC differently by taking the policy and termination function corresponding to an option $o^{1:j}$ as a single coagent $\kappa_{o^{1:j}}$ (\cref{treehocfigure}). This coagent will act as $\pi^{j+1}(\cdot|s_{o,j})$ when given a state with execution mode $d$, and act as the termination function $\beta_j(s_{o,j})$ given a state with execution mode $u$. In this case, the graph of the CN is a tree where each edge is doubled, with one edge going up and another going down, and each vertex has a loop. Notice this graph is a contraction of \cref{treehocfigure1} along the edges connecting the two copies of the binary tree.

\begin{figure}[h]
    \centering
    \includegraphics[width=\columnwidth]{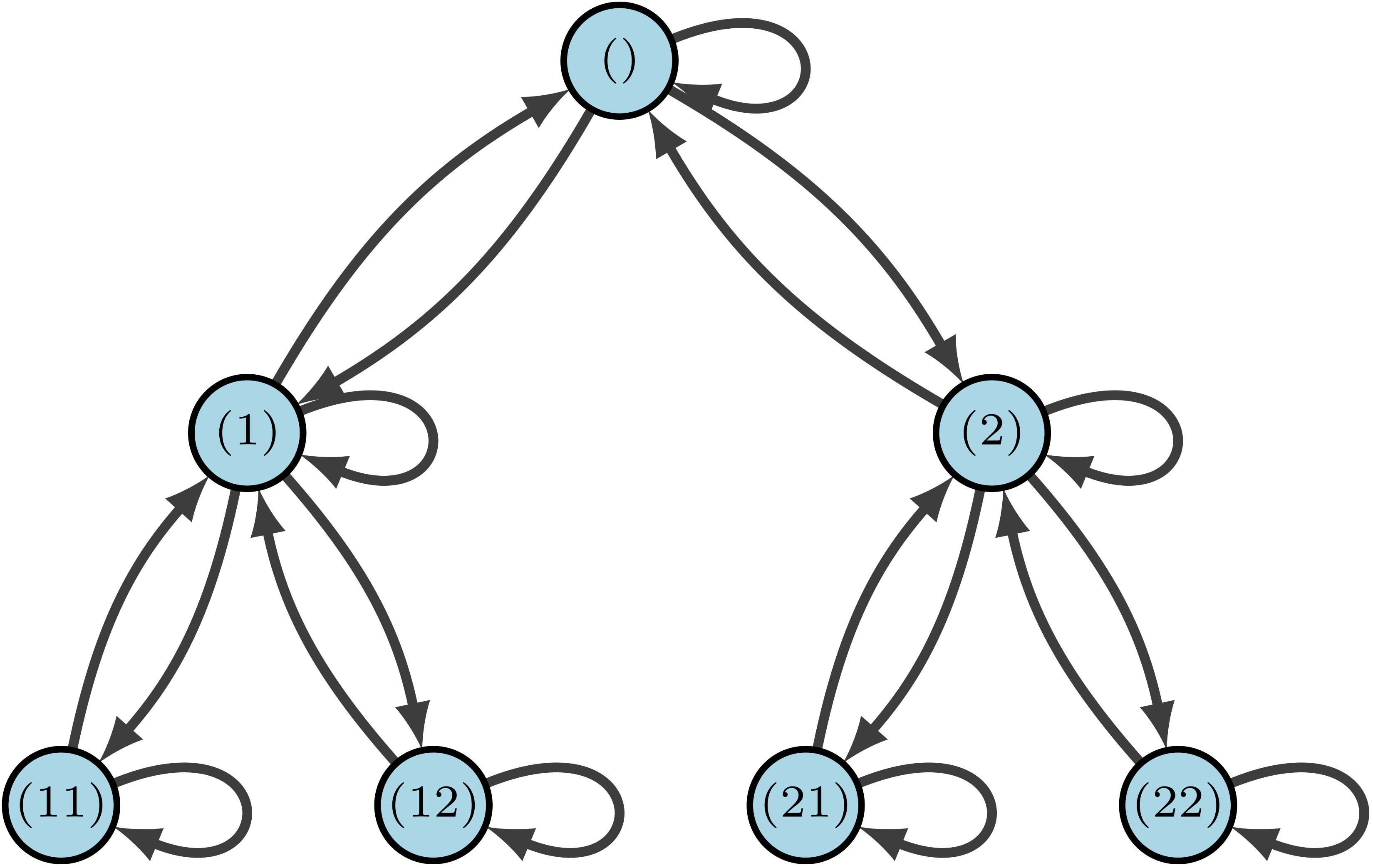}
    \caption{A graph representing the coagent network of HOC with $N=3$ and $m_i=2$. The policy and termination function of each option form a single coagent. Nodes are labelled by their unique address.}
    \label{treehocfigure}
\end{figure}

\textbf{Bellman equations in HOC. }To express the Bellman equations that will be necessary to understand the equivalence of our PGT for HOC and the one in the literature, we first need to clarify our notation for the edge cases:
\begin{itemize}
    \item The state $s_{o,N,d}$ is always followed by some state $s'_{o,N,u}$ for some $s' \in \mcS_{\env}$. Notice it is always the case that a primitive action such as $o^N$, if it were to be informally looked at as an option, immediately terminates after taking place, in other words $\beta_N()\equiv 1$. So $s'_{o,N,u}$ is \textit{immediately} followed by $s'_{o,N-1,u}$. Therefore, it is meaningful to take $s'_{o,N,u} = s'_{o,N-1,u}$.
    \item  $s_{o,0,d}$ or $s_{o,0,u}$ means the highest level of hierarchy $o^0$ at state $s$, where the first option has not been chosen yet. The highest level policy never terminates. Therefore, as $\beta_0()\equiv 0$, $s_{o,0,u}$ is \textit{immediately} followed by $s_{o,0,d}$ and as such $s_{o,0,u} = s_{o,0,d}$. This state is followed by $s_{o,1,d}$ for some option $o^1$ assigned by $\pi^1$. 
    \item  In our summations $\sum_{x}^y$ or products $\prod_{x}^y$, if the higher index is smaller than the lower index ($x>y$), we consider the result to be $1$. This convention is helpful as it makes the equations easier to express without specifically accounting for the edge cases.
\end{itemize}

To describe $Q-$values of HOC, recall that in standard RL, we tend to differentiate the action value and the state value by using a different notation ($Q$ and $V$). Our terminology allows us to unambiguously use the same notation $Q()$ for the value of our augmented states at any level of hierarchy or any state of execution $d$ or $u$. 

First, let $r(s_{o,N})$ be the reward (or expected reward, if the environment reward process is stochastic) given by the environment to the agent which has taken primitive action $o^N$ at $s\in \mcS_{\env}$. Let $\gamma \in [0,1]$ be the discount factor. All $Q-$values will be computed in succession. Starting with the easiest case at the lowest level, $Q(s_{o,N,d})=$
\begin{align}
    r(s_{o,N})+\gamma \sum_{s'}P(s'_{o,N-1,u} | s_{o,N,d})Q(s'_{o,N-1,u})
\end{align}
where $P$ is the transition probability of the environment, that given action $o^N$ at $s$, the state changes to $s'$. Note that $Q(s_{o,N,d}) = Q(s_{o,N-1,d},o^N)$, i.e. the $Q$-value of the state-action pair $(s_{o,N-1,d},o^N)$ for the policy $\pi^N$ of the option $o^{1:N-1}$.

The equation above is obvious, as the next state is $s'_{o,N,u} = s'_{o,N-1,u}$, which is in the upward execution phase. Next, computing the value at this state, $Q(s_{o,N-1,u})=$
\begin{align}\label{bellman_u_N-1}
   \sum_{i=0}^{N-1} (1-\beta_i(s_{o,i})) \big (\prod_{j=i+1}^{N-1} \beta_j(s_{o,j})\big ) Q(s_{o,i,d})
\end{align}
Note that the product $(1-\beta_i(s_{o,i})) \big (\prod_{j=i+1}^{N-1} \beta_j(s_{o,j})\big )$ is the probability of upward trajectory up to level $i+1$. At that level, the decision is to not terminate, which has probability $1-\beta_i(s_{o,i})$. Notice $Q(s_{o,N-1,u})$ can also be interpreted as the state-value for termination $\beta_{N-1}$ of option $o^{1:N-1}$ at state $s_{o,N-1,u}$. Similar interpretations exist for the $Q$-values computed later.

Finally, computing $Q(s_{o,i,d})$ completes the cycle needed to compute $Q(s_{o,N,d})$, as a `Bellman' equation is a recursive equation. To do so, observe the following:
\begin{align}\label{bellman_f}
Q(s_{o,l-1,d})=\sum_{o^l}\pi^l(o^l|s_{o,l-1})Q(s_{o,l,d}), \  \forall l\le N 
\end{align}
Applied iteratively, one obtains an expression of $Q(s_{o,i,d})$ in terms of $Q(s_{o,N,d})$, which was already computed.

In similar fashion, one can compute the $Q$ value at all other states. To do so, we shall derive the rewards for any node using \cref{cumrewardpiodfn}:
\begin{align}
    r(s_{o,l})= \sum_{o^{l+1:N}} \prod_{i=l+1}^N \pi^{i}(o^i|s_{o,i-1})r(s_{o,N}), \ \forall l\le N
\end{align}
Then, using the transition probability $P_\pi(s'_{o,N,u} | s_{o,l,d})=$
\begin{align}\label{transition_prob}
     \Big (\prod_{i=l+1}^N \pi^{i}(o^i|s_{o,i-1})\Big )P(s'|s_{o,N}) , \ \forall l \le N
\end{align}
it follows
\begin{align}\label{bellman_d}
    Q(s_{o,l,d})=r(s_{o,l})+\gamma \sum_{s',o^{l+1:N}}P_\pi(s'_{o,N,u} | s_{o,l,d})Q(s'_{o,N,u})
\end{align}
Notice in both equations above, $s'_{o,N,u}$ can be replaced by $s'_{o,N-1,u}$. Further, for all states with upward execution mode, one can write $Q(s_{o,l,u})=$
\begin{align}\label{bellman_u}
    \sum_{i=0}^l (1-\beta_i(s_{o,i})) \big (\prod_{j=i+1}^{l} \beta_j(s_{o,j})\big ) Q(s_{o,i,d}), \ \forall l\le N
\end{align}
When the algorithm is at $s_{o,l,u}$, it will begin executing upward and reach some state $s_{o,i,u}$, for some $0\le i\le l$ (with probability $\prod_{j=i+1}^{l}\beta_j(s_{o,j})$), where it decides to execute downward with probability $(1-\beta_i(s_{o,i})$), i.e. the state changes to $s_{o,i,d}$. The above equation multiplies these probabilities by the value $Q(s_{o,i,d})$ of the state $s_{o,i,d}$ it starts executing downward at.

\begin{rmk}
In the literature \cite{sutton1999between,bacon2017option,riemer2018learning}, $Q_\Omega(s,o^{1:i})$ generally corresponds to our $Q(s_{o,i,d})$ while $Q_U(s,o^{1:i})$ generally corresponds to $Q(s_{o,i,u})$; in the edge case of $i=N$, this correspondence could become a bit imprecise, due to the fact that the literature notation accounts for those cases separately by introducing new notations. Finally, $V_\Omega(s)$ corresponds to $Q(s_{o,0,d})$.
\end{rmk}

\begin{rmk}\label{bacon_correspondence}
The equations (\ref{bellman_d}) and (\ref{bellman_f}) give the following in the case of $l=2,N=2$:
\begin{align}\label{bellman_d_l=2_N=2}
Q(s_{o,1,d})&=\sum_{o^2}\pi^2(o^2|s_{o,1})Q(s_{o,2,d})\\
    Q(s_{o,2,d})&=r(s_{o,2})+\gamma \sum_{s'}P(s'_{o,1,u} | s_{o,2,d})Q(s'_{o,1,u})
\end{align}
which can be recognized as the equations for $Q_\Omega$ in terms of $Q_U,U$ as written in \cite[Eqs. 1-2]{bacon2017option}.
\end{rmk}

\textbf{Policy gradient theorem for HOC. }Below, we write the PGT theorem for HOC to motivate the next notations and probability computations. In a series of steps, we will identify this PGT with \cite[Theorem 1]{riemer2020role}. For an initial state $(s_{0})_{o_0,N-1,d}$ at the lowest level of the tree at node $o_0^{1:N-1}$, we have
\begin{thm}\label{HOCPGTRewritten}
$\frac{dQ}{d\theta}((s_{0})_{o_0,N-1,d})=$
\begin{align}
    &\sum_{s,o^{1:N-1},s'}\mu(s',s_{o,N-1,d}|(s_{0})_{o_0,N-1,d})\nonumber \\
    &\Big[ \sum_{o^N}\frac{d\pi^N}{d\theta}(o^N|s_{o,N-1})Q(s_{o,N-1,d},o^N) \nonumber \\
    &- \gamma \sum_{l=1}^{N-1}\Big(\prod_{k=l+1}^{N-1}\beta_k(s'_{o,k})\Big)\frac{d\beta_l}{d\theta}(s'_{o,l}) \cdot \nonumber \\
    & (Q_{\beta_l}(s_{o,l},0) - Q_{\beta_l}(s_{o,l},1)) \nonumber \\
    &+\gamma \sum_{j=0}^{N-1}\Big(\prod_{k=j}^{N-1}\beta_k(s'_{o,k})\Big)\sum_{o'^{1:j}}P_{\beta,\pi}(s'_{o',j-1,d}|s'_{o,j-1,u})  \cdot \nonumber \\ \label{policyimpsamplingterm}
    &\frac{d\pi^j}{d\theta}(o'^{j}|s'_{o',j-1})Q(s'_{o',j-1,d},o'^j) \Big].
\end{align}
\end{thm}
There are a few notations that need to be defined:
\begin{itemize}
    \item $\mu(s', s_{o,N-1,d}|(s_{0})_{o_0,N-1,d})$ is the discounted probability of starting at $(s_{0})_{o_0,N-1,d}$, reaching the state $s_{o,N-1,d}$, which after $\pi^N$ execution is followed by $s'_{o,N-1,u}$. It does not directly correspond to $d(x_o,x|x_0)$ in \cref{CNPGT}, but it will be shown to be related. 
    \item $P_{\beta,\pi}(s'_{o',j-1,d}|s'_{o,j-1,u})$ is the transition probability from $s'_{o,j-1,u}$ to $s'_{o',j-1,d}$, which is dependent on the termination functions $\beta$ and policies $\pi$. It is the probability of the algorithm reaching $s'_{o',j-1,d}$ at level $j$, after finishing its upward phase of execution (which started at the same level at $s'_{o,j-1,u}$) and going down again to level $j$ (see \cref{P_beta_dfn_t}).
    \item $Q_{\beta_l}(s_{o,l},0), Q_{\beta_l}(s_{o,l},1)$ are the action value functions of $\beta_l$ for option $o^{1:l}$ at $s_{o,l}$. Action $1$ denotes termination and $0$ denotes the opposite (see \cref{advantage_2}).
\end{itemize}

The notations that match with the reference have also the same meaning, for example $P_{\beta,\pi}(\cdot|\cdot), \mu(\cdot|\cdot)$ or the advantage $A()$. As will be shown in equations (\ref{rewrite_pi^N}-\ref{rewrite_pi^l}), we can interpret the above sum as the sum of the policy gradients of $\pi^j$ and $\beta_j$ of the options, as their similarity to the usual policy gradient $\sum_x d(x|x_0) \sum_a \frac{d\pi}{d\theta}(a|x) Q(x,a)$ in (\ref{policygradient}) can be observed. Below we will make steps towards proving this interpretation.


Let us rewrite (\ref{bellman_u}) so that the $Q-$value on the right-hand side is evaluated at states with the same level as the left-hand side. To do so, apply (\ref{bellman_f}) repeatedly on $Q(s_{o,i,d})$ to get to $Q(s_{o,l,d})$, which leads to the following when $l\le N$:
\begin{align}\label{bellman_u2}
    Q(s_{o,l,u})=&\sum_{o'^{1:l}} \Big (\sum_{i=0}^l (1-\beta_i(s_{o,i})) \big (\prod_{j=i+1}^{l} \beta_j(s_{o,j})\big ) \nonumber \\
    &\big (\prod_{p=i+1}^{l} \pi^p(o'^p|s_{o',p-1})\big ) 1_{o'^{1:i}=o^{1:i}}\Big ) Q(s_{o',l,d}),
\end{align}
where $1_{o'^{1:i}=o^{1:i}}$  is one if $o'^{1:i}=o^{1:i}$ and zero otherwise. The term multiplied by $Q(s_{o',l,d})$ has an independent meaning, which will allow to make this equation more compact. Let $P_{\beta,\pi}(s_{o',l,d}|s_{o,l,u})$ be the transition probability from $s_{o,l,u}$ to $s_{o',l,d}$, which is dependent on the termination functions $\beta$ and policies $\pi$. Hence, this probability solely depends on the model and not the environment. It is the probability of the algorithm executing downward at $o'^{1:l}$ at level $l+1$, after finishing its upward phase of execution (which started at level $l+1$) and coming down again to level $l+1$. It is easy to see that this probability is the same multiplicative term in the equation above:
\begin{align}\label{P_beta_dfn_t}
    P_{\beta,\pi}(s_{o',l,d}|s_{o,l,u}) =& \sum_{i=0}^l (1-\beta_i(s_{o,i})) \big (\prod_{j=i+1}^{l} \beta_j(s_{o,j})\big ) \nonumber \\
    &\big (\prod_{p=i+1}^{l} \pi^p(o'^p|s_{o',p-1})\big ) 1_{o'^{1:i}=o^{1:i}} .
\end{align}
One can generalize this to compute other useful transition probabilities $\forall  l \le m \le N$:
\begin{align}\label{p_beta_gen1}
    &P_{\beta,\pi}(s_{o',m,d}|s_{o,l,u}) = \nonumber \\
    &\prod_{p=l+1}^m \pi^p(o'^p|s_{o',p-1}) P_{\beta,\pi}(s_{o',l,d}|s_{o,l,u}) ,\\ \label{p_beta_gen2}
    &P_{\beta,\pi}(s_{o',m,d}|s_{o,l,u}) = \nonumber \\
    &P_{\beta,\pi}(s_{o',m,d}|s_{o,m,u})\prod_{j=m+1}^l \beta_j(s_{o,j}).
\end{align}
Using the above, equation (\ref{bellman_u2}) becomes:
\begin{align}\label{bellman_u_final}
    Q(s_{o,l,u})=\sum_{o'^{1:l}}P_{\beta,\pi}(s_{o',l,d}|s_{o,l,u}) Q(s_{o',l,d}).
\end{align}
Next, let us define $\mu(s'_{o',l-1,d}|s_{o,l-1,d})$, which is the sum of discounted probabilities for getting from $s_{o,l-1,d}$ to $s'_{o',l-1,d}$. We will compute two transition probabilities. Using (\ref{transition_prob}) to write the transition probability for $s_{o,l,d} \to s'_{o,N-1,u}$:
\begin{align}
   &P_\pi(s'_{o,N-1,u}|s_{o,l,d})= P_\pi(s'_{o,N,u}|s_{o,l,d})= \nonumber\\
   &\sum_{o^{l+1:N}} P(s'|s_{o,N}) \prod_{i=l+1}^N \pi^{i}(o^i|s_{o,i-1})
\end{align}
Also, for $s_{o,l-1,d} \to s'_{o,l-1,u}$:
\begin{align}\label{p_transition_l-1}
    &P_{\beta,\pi}(s'_{o,l-1,u}|s_{o,l-1,d})= \nonumber \\
    &\sum_{o^{l:N-1}} \prod_{j=l}^{N-1} \beta_j(s'_{o,j}) P_\pi(s'_{o,N-1,u}|s_{o,l-1,d})
\end{align}
Next, define the discounted one-step and $k$-steps recursively as follows $\forall l\le N$:
\begin{align}\label{p_gamma_dfn_d}
    &P_\gamma^{(0)}(s'_{o',l-1,d}|s_{o,l-1,d}) = 1_{s'_{o',l-1,d} = s_{o,l-1,d}},   \nonumber \\
    &P_\gamma^{(1)}(s'_{o',l-1,d}|s_{o,l-1,d})= \nonumber \\
    &\gamma P_{\beta,\pi}(s'_{o',l-1,d}|s'_{o,l-1,u}) P_{\beta,\pi}(s'_{o,l-1,u}|s_{o,l-1,d}), \nonumber  \\
    &P_\gamma^{(k)}(s'_{o',l-1,d}|s_{o,l-1,d})= \nonumber \\
    &\sum_{s''_{o'',l-1}} P_\gamma^{(k-1)}(s'_{o',l-1,d}|s''_{o'',l-1,d})P_\gamma^{(1)}(s''_{o'',l-1,d}|s_{o,l-1,d}).
\end{align}
The definition for $\mu()$ can be written for different transitions. The one needed for HOCPGT is:
\begin{align}\label{mu2}
    \mu(s'_{o',l-1,d}|s_{o,l-1,d})=\sum_{k=0}^\infty P_\gamma^{(k)}(s'_{o',l-1,d}|s_{o,l-1,d}), 
\end{align}
holding $\forall l \le N$. Let us now define the  \textit{advantage} as follows:
\begin{align}\label{advantage}
    A(s_{o,l})=Q(s_{o,l,d}) - Q(s_{o,l-1,u})
\end{align}
The advantage answers this question: If one is able to choose, then how much advantageous it is to start executing downward from $s_{o,l}$, than to change the higher level options and try a different set of options $s_{o',l}$? The difference above in the values determines the advantage of this choice. Indeed, just like the first term $Q(s_{o,l,d})$ is the value of not terminating, the second term is the value of terminating. Thus, alternatively, one can write:
\begin{align}\label{advantage_2}
    A(s_{o,l})= Q_{\beta_l}(s_{o,l},0) - Q_{\beta_l}(s_{o,l},1)
\end{align}
using the action value functions of $\beta_l$. It is clear that
\begin{align}\label{beta_q_values}
    Q_{\beta_l}(s_{o,l},1) = Q(s_{o,l-1,u}) \ , \    Q_{\beta_l}(s_{o,l},0) = Q(s_{o,l,d}),
\end{align}
where $1$ denotes the termination action, and $0$ the opposite action. These actions have probabilities:
\begin{align}\label{beta_probs}
    \beta_l(1|s_{o,l}) := \beta_l(s_{o,l}) \ , \ \beta_l(0|s_{o,l}) :=1-\beta_l(s_{o,l}).
\end{align}
We are now ready to state the rewritten HOCPGT as written by \cite{riemer2020role}.
\begin{thm}\cite[Theorem 1]{riemer2020role}\label{HOCPGT}
$\frac{dQ}{d\theta}((s_{0})_{o_0,N-1,d})=$
\begin{align}
\begin{split}
    &\sum_{s,o^{1:N-1},s'}\mu(s',s_{o,N-1,d}|(s_{0})_{o_0,N-1,d}) \\
    &\Big[ \sum_{o^N}\frac{d\pi^N}{d\theta}(o^N|s_{o,N-1})Q(s_{o,N,d})\\
    &- \gamma \sum_{l=1}^{N-1}\Big(\prod_{k=l+1}^{N-1}\beta_k(s'_{o,k})\Big)\frac{d\beta_l}{d\theta}(s'_{o,l})A(s'_{o,l})\\
    &+\gamma \sum_{j=0}^{N-1}\Big(\prod_{k=j}^{N-1}\beta_k(s'_{o,k})\Big)\\
    &\sum_{o'^{1:j}}P_{\beta,\pi}(s'_{o',j-1,d}|s'_{o,j-1,u})\frac{d\pi^j}{d\theta}(o'^{j}|s'_{o',j-1})Q(s'_{o',j,d}) \Big]
\end{split}
\end{align}
where $\mu(s',s_{o,N-1,d}|(s_{0})_{o_0,N-1,d})$ is the discounted probability of starting at $(s_{0})_{o_0,N-1,d}$, reaching the state $s_{o,N-1,d}$, which after $\pi^N$ execution is followed by $s'_{o,N-1,u}$.
\end{thm}
\begin{rmk}
The last term above might seem different from the one in the reference \cite[Theorem 1]{riemer2020role}, but they are the same. Our convention, motivated from our point of view, is to sum over each node $o'^{1:j}$ individually, while in the reference, the last node $o'^{1:N-1}$ at the bottom of the tree is taken and the sum is over its parents.
\end{rmk}
We now show how to identify our main theorem application to HOC with the above. Let us denote  $x_0 = (s_{0})_{o_0,N-1,d}$ and $x = s_{o,N-1,d}$. Recall the HOC policy is denoted by $\Pi$. Then the first term in the bracket above can be written as:
\begin{align}\label{rewrite_pi^N}
    \sum_{x} d^\Pi(x|x_0) \sum_{o^N} \frac{d\pi^N}{d\theta}(o^N|x)Q_{o^{1:N-1}}(x,o^N),
\end{align}
where $d^\Pi(x|x_0)$ is the discounted transition probability for a policy $\Pi$, for reaching $x$ from $x_0$. Notice there is no $s'$, as $\frac{d\pi^N}{d\theta}(o^N|s_{o,N-1})Q(s_{o,N,d})$ does not depend on $s'$, and so the summation over $s'$ averages out this outcome state. Also note that $Q_{o^{1:N-1}}(x,o^N)$ is the $Q$-value of $\pi^N(\cdot|o^{1:N-1})$ policy of the node $o^{1:N-1}$, hence equal to $Q(s_{o,N,d})$. Therefore, the first term is simply the policy gradient of the lowest level node.

For the second term, for a given node $o^{1:l}$ for $1 \le l \le N-1$, the term $\gamma \Big(\prod_{k=l+1}^{N-1}\beta_k(s'_{o,k})\Big)$ can be absorbed into the discounted transition probability $\mu()$ to create $d^\Pi(x'_{o,l},x|x_0)$ where $x'_{o,l} = s'_{o,l,u}$. This is the discounted transition probability from $x_0 \to x$ in any number of execution paths, and from $x \to x'_{o,l}$ within the same execution path. 
\begin{rmk}\label{path_exec_HOC}
Notice that an execution path in this setting, is an execution from $\pi^N$ followed by a path upward, then a path downward until before the next execution of $\pi^N$ takes place. Thus, as the path involves an environmental step due to the execution of $\pi^N$ at the beginning of the path, the discounted transition probability takes a discount factor $\gamma$. 
\end{rmk}
The above remark applies when going from $d^\Pi(x|x_0)$ to $d^\Pi(x'_{o,l},x|x_0)$, as is the case in $\gamma \Big(\prod_{k=l+1}^{N-1}\beta_k(s'_{o,k})\Big)$. Furthermore, using (\ref{beta_probs}),
$$
    \frac{d\beta_l}{d\theta}(1|s'_{o,l}) =  \frac{d\beta_l}{d\theta}(s'_{o,l}) \ ,  \ \frac{d\beta_l}{d\theta}(0|s'_{o,l}) = - \frac{d\beta_l}{d\theta}(s'_{o,l}).
$$
From the above, using (\ref{advantage_2}) and (\ref{beta_q_values}):
\begin{align}\label{rewrite_beta^l}
\begin{split}
    &- \sum_{s,o^{1:N-1},s'}\mu(s',s_{o,N-1,d}|(s_{0})_{o_0,N-1,d}) \\
    &\gamma \Big(\prod_{k=l+1}^{N-1}\beta_k(s'_{o,k})\Big)\frac{d\beta_l}{d\theta}(s'_{o,l})A(s'_{o,l}) =   \\ 
    &\sum_{x'_{o,l}, x} d^\Pi(x'_{o,l},x|x_0) \sum_{a=0,1} \frac{d\beta_l}{d\theta}(a|x'_{o,l}) Q_{\beta_l}(x'_{o,l},a).
\end{split}
\end{align}
Similar to (\ref{rewrite_pi^N}), one can see that (\ref{rewrite_beta^l}) is the usual policy gradient theorem for the termination policies $\beta_l$ at $o^{1:l}$. Finally, for the last term, for each $\pi^j$ at $o'^{1:j-1}$:
\begin{align}\label{rewrite_pi^l}
\begin{split}
    &\sum_{s,o^{1:N-1},s'}\mu(s',s_{o,N-1,d}|(s_{0})_{o_0,N-1,d}) \\
    &\Big[\gamma \Big(\prod_{k=j}^{N-1}\beta_k(s'_{o,k})\Big)  P_{\beta,\pi}(s'_{o',j-1,d}|s'_{o,j-1,u}) \times \\
    &\sum_{o'^j}\frac{d\pi^j}{d\theta}(o'^{j}|s'_{o',j-1})Q(s'_{o',j,d})\Big ] = \\ 
    &\sum_{x'_{o',j-1},x} d^\Pi(x'_{o',j-1},x|x_0) \\
    &\sum_{o'^j}\frac{d\pi^j}{d\theta}(o'^{j}|x'_{o',j-1})Q_{o'^{1:j-1}}(x'_{o',j-1}, o'^j)
\end{split}
\end{align}
Where $d^\Pi(x'_{o',j-1},x|x_0)$ is the discounted transition probability from $x_0 \to x$ in any number of execution paths, and from $x \to x'_{o',j-1}= s'_{o',j-1}$ within the same execution path. The factor $\gamma \Big(\prod_{k=j}^{N-1}\beta_k(s'_{o,k})\Big) P_{\beta,\pi}(s'_{o',j-1,d}|s'_{o,j-1,u})$ is what describes the discounted probability of this last transition within the same execution path.
\begin{rmk}
Notice our proof does not involve any further state-augmentation. This is in contrast to the proof in \cite[App. C]{riemer2020role}, where something called the termination vector $T$ is defined and the graph structure is completely forgotten to be able to cast HOC into a synchronous coagent network (only two coagents). This shows that our framework can provide more intuitive ways for computing the policy gradient of coagent networks.
\end{rmk}

\section{Non-Examples of Coagent Network}\label{nonexamples_cn}
\textbf{Attention Option-Critic (AOC). }The AOC model \cite{chunduru2020attention} is an OC augmented with attention mechanisms $h_\omega$ for each option $\omega$. The state $s_\omega$, input for $\pi_\omega$ and its termination function $\beta_\omega$, is obtained as follows:
\begin{itemize}
    \item Take the environment state $s$ along with the address of the option and the mode of execution (like $s_{o,2,d}$),
    \item change the environment state $s$ to $s_\omega = h_\omega(s) \odot s$ where $h_\omega(s)$ has the same shape as $s$ with values in $[0,1]$, and $\odot$ is the Hadamard multiplication.
\end{itemize} 
The motivation behind this is to focus each option on parts of the states that are relevant. This model leverages attention to show that it is possible to achieve useful options with \textit{reasonable} lengths. 

The AOC learning algorithm can not be recast as a policy gradient of a CN. The main obstacle is the attention mechanism, which is \textit{deterministic}, and thus is not a coagent. It is learned using gradients of the option action value. 

However, following Remark \ref{makingstochastic}, if the attention mechanisms were stochastic, then we could apply the PGT as in Theorem \ref{CNPGT}. The coagents in this case are the policies, terminations and the attention mechanisms. The rule $\Delta$ is similar to that of HOC, where we let the attention corresponding to an option choose an action when execution mode is $d$. 

\textbf{Hierarchical Actor-Critic (HAC). }From the point of view of \cref{CNdfn}, there are two ways that HAC \cite{levy2017hierarchical} fails to be a CN. 

First, we recall how HAC functions. HAC is a hierarchical goal-based model of policies $\pi^1,\ldots,\pi^N$ where each policy selects a goal $g_i$ to be achieved by $\pi^{i+1}$. Each policy has a fixed time horizon of $K$ steps in which it has the chance to achieve the goal. It terminates once it reaches the goal or after $K$ many executions, and the parent selects a new goal. As a result, the policy $\pi^i$ acts again after at most $K^{N-i}$ many time-steps. The reward given to a policy is based on how close the current state is to the goal, which is measured by a metric defined by the user. The edge case of $g_0$ (the goal for $\pi^1$) is either set by the user or simply not chosen, in which case the reward for $\pi^1$ is just the environmental reward. In practice, the reward for $\pi^i$ is a linear combination of the goal reward and the environmental reward. Clearly, for the more abstract policies the environmental reward should have more weight. In any case, it is clear that the policies do not have a single reward source (\cref{rmk:singlereward}).

Another CN assumption that HAC violates is the Markov property. This difference can be resolved, at least in theory. If one considers, say, the mid-level policy of a $3-$level HAC, its action depends on the subgoal $g_1$ it has received from the highest policy for a fixed $K^2$ executions paths (environmental time steps) or until it reaches the subgoal. During each $K$ executions of the $K^2$ executions, the mid-level policy is only called once by the lowest policy trying to reach the subgoal $g_2$ assigned by the mid-policy during that $K$ time-steps. The mid-policy has to remember the subgoal $g_1$ it was assigned to in the first execution path for the next $K^2$ execution paths. Note that such non-Markovianness is also present in the case of asynchronous firing in \cite{kostas2020asynchronous} as remarked in the main text. This is not an insurmountable theoretical barrier in viewing these models as CNs, at least if one is willing to augment the state space by including all histories. In the case of goal-conditioned models, one can provide the information of all selected subgoals into the state. 
Thus, more generally, by the same Markov trick, the Markov condition in \cref{CNdfn} can be lifted.

\section{More on Examples of Coagent Networks}\label{appsec:more_on_examples} 


The freedom in the graph structure and parameter sharing pattern, could allow us to exploit relational biases in a given environment/task to build more powerful RL agents. We mention a general but customizable coagent network, followed by two other examples in the literature.

\textbf{1) Coagent networks with cycles or loops.} Let us consider a cycle for the CN graph and design a corresponding rule $\Delta$. The nodes are labeled by $o_n$ for $n \in \{1, \ldots, N\}$, where each node, depending on the input, makes a primitive action or sends some information via its action $u_n$ to the next node (with $o_N$ sending to $o_1$). The description of the rule $\Delta$ is as follows: one assigns each state $s$ of the environment two integers $n_s \in \{1,\ldots,N\}$ and $l_s \in \mbbN$, which means the node $o_{n_s}$ is the one that performs a primitive action at environment state $s$ after $s$ has passed through $l_s$ many cycles. More generally, one can imagine $N$ decision functions $f_{n,s} : \mcS_{\pi_{o_n}} \to \{1, \ 0\}$ for any $s\in \mcS_{\env}$ and $n \in \{1, \ldots, N\}$, where $1$ means primitive and $0$ non-primitive. Depending on the input of the node, after finitely many coagent executions, we have $f_{n,s}(x_{o_n}) = 1$. This should be guaranteed by the programmer to happen to avoid infinite loops. This forces the node $o_n$ to apply a primitive action at $x_{o_n}$, which may contain $s$ as the information regarding the environment.

\textbf{2) Options of Interest. }One of the structures that one can add to the HOC is an interest function $I_\omega$ for each option $\omega$ \cite{khetarpal2020options}. One interpretation of this function is that it introduces a deterministic part in the process of the selection of $\omega$ by its parent $\Omega$, which is unique to $\omega$ and disjoint from the process of selection of the other children. This modularity has an effect similar to attention in Attention Option-Critic, discussed in \cref{nonexamples_cn}. We can still apply \cref{CNPGT} by combining the parent with the interest functions of its children, i.e. the new coagent being $\pi^{new}_{\Omega} (\omega|s) = I_\omega(s) \pi^{old}_{\Omega} (\omega|s) $. Then \cref{CNPGT} applied to $\{\pi_o^{new}, \beta_o\}_{\text{options }o}$ will give the same PGT in \cite{khetarpal2020options}.

\textbf{3) Stochastic Neural Nets. }Stochastic Neural Nets, in their full generality, have stochastic activation functions and/or weights. In all cases, each neuron forms a coagent with its parameters being the parameters of the activation function along with the weights connected to the neuron that compute its input. A more refined decomposition of the agent to coagents is possible if the weights are stochastic, in which case, every weight is also a coagent by itself. Notice that although our discussion has been mainly around RL applications, we can apply our results to the supervised learning setting by simply using a discount factor $\gamma = 0$, although we naturally expect supervised learning algorithms to be superior in their own field of application.
\begin{rmk}\label{makingstochastic}
As a general idea, one could inject stochasticity into deterministic (parts of) RL models and apply the coagent network policy gradient theorem, and potentially obtain better results, as demonstrated in the noisy DQN model \cite{fortunato2017noisy}.
\end{rmk}


\section{Synchronous Coagent Network PGT}\label{appssec:synchronous_full_proof}
We finish the proof by explaining the marginalizations mentioned after \cref{eq:synch_before_marginalization}. We first note that the marginalization of $Q_{\Pi_{i}}(X_{i},U_{i})$ over $(u_{i,o'})_{o'\neq o}$ given by the weighting $\prod_{o'\neq o} \pi_{i,o'}(u_{i,o'}|x_{i,o'})$, gives the $Q$-value $Q(X_i,u_{i,o})$ as the latter is naturally defined as $\sum_{U_i}\prod_{o'\neq o} \pi_{i,o'}(u_{i,o'}|x_{i,o'})Q_{\Pi_{i}}(X_{i},U_{i})$, where the sum is over all $U_i$ such that $o$'s action is $u_{i,o}$. However, according to PGT, we should compute $Q_{\pi_{i,o}}(x_{i,o}, u_{i,o})$, which can be done as follows
\begin{align}
    Q_{\pi_{i,o}}(x_{i,o}, u_{i,o}) = \sum_{X_i} P( X_i|x_{i,o})Q(X_i,u_{i,o}).
\end{align}
The probability $P(X_i|x_{i,o})$ is the probability that within the same time step of $x_{i,o}$, the state of $\Pi_i$ is $X_i$. This probability is dependent on the environment MDP and $\Pi$. One can recover this probability from the discounted state action occupancy as follows:
\begin{align}
\begin{split}
     \sum_{x_0,x} d(x_0) d(X_i,x|x_0) &= \sum_{x_0,x} d(X_i,x,x_0 ) = d(X_i) \\
     &= d(x_{i,o}, (x_{i,o'})_{o'\neq o}) \\
     &= d(x_{i,o}) P((x_{i,o'})_{o'\neq o}|x_{i,o}) \\
     &=  d(x_{i,o}) P(X_i|x_{i,o}).
\end{split}
\end{align}
The last two equalities require further explanation. We recall the definition of the discounted state action occupancy
\begin{align}
\begin{split}
    &d(x_{i,o}, (x_{i,o'})_{o'\neq o}) = \\ 
    &\sum_{k=1}^\infty \gamma^k P^{(k)}(x_{i,o}, (x_{i,o'})_{o'\neq o}) = \\
    &\sum_{k=1}^\infty \gamma^k P^{(k)}(x_{i,o}) P^{(k)}((x_{i,o'})_{o'\neq o}|x_{i,o}) 
\end{split}
\end{align}
Here $P^{(k)}$ is the state action occupancy corresponding to the $k$-th step, its value depending on $k$, $\Pi$, and the MDP $\mcM$. We need to show $P^{(k)}((x_{i,o'})_{o'\neq o}|x_{i,o}) = P((x_{i,o'})_{o'\neq o}|x_{i,o})$, i.e. that it is independent of $k$. Notice that the policy $\Pi$'s internal actions and outputs does not depend on $k$, but on $s$; and while the distribution of $s$ itself indeed depends on ($\Pi$, $\mcM$, $k$), we note that the state $s$ is already given in $x_{i,o}=(U_{i,o}^\pre,s)$. Thus, the term $P^{(k)}((x_{i,o'})_{o'\neq o}|x_{i,o}) $ is independent of $k$.

Finally, after marginalizing $Q(X_i,u_{i,o})$ over $(x_{i,o'})_{o'\neq o}$, we reintroduce the one over $x_0,x$:
\begin{align}
\begin{split}
        &\sum_{x_0 \in \mcS_{\Pi,\op{init}}} d(x_0) \sum_{i,o} \sum_{x\in S_{\Pi}, x_{i,o} \in \mcS_{\pi_{i,o}}} d(x_{i,o}, x|x_0)  \\
        &\sum_{u_{i,o} \in \mcA_{x_{i,o}}}  \frac{d\pi_{i,o}}{d\theta}(u_{i,o}|x_{i,o})Q_{\pi_{i,o}}(x_{i,o},u_{i,o}).
\end{split}
\end{align}

\section{Asynchronous Coagent Network PGT}\label{appssec:asynchronous_full_proof}

For thoroughness, we explicitly state the changes that one has to make to the argument in \cite{kostas2020asynchronous} in order to show that the PGT of the synchronous CN obtained after state augmentation simplifies to the the asynchronous coagents policy gradients. As mentioned in the main text, the argument around this has no relevance to parameter sharing. Therefore, as illustrated below, we note that we simply have replaced the nonshared parameters $\theta_i$ by $\theta$ in the argument in \cite[p. 7]{kostas2020asynchronous}, and summed over $i$ for the policy gradients $\nabla J, \nabla \grave J$:
\begin{quote}
``[...] Having shown that the expected return in the asynchronous setting is equal to the expected return in the synchronous setting, we turn to deriving the asynchronous local policy gradient, $\Delta_i$. It follows from $J(\theta) = \grave J(\theta)$ that $\nabla J(\theta) = \nabla \grave  J(\theta)$. Since $\grave \pi$ is a synchronous, acylic network, and $\grave M$ is an MDP, we can apply the CPGT to find an expression for  $\nabla \grave  J(\theta)$.
For [shared parameter] synchronous network, this gives us $\frac{\partial \grave J(\theta)}{\partial \theta} =$
$$
     \sum_i \mathbb{E} \Big[ \sum_{t=0}^\infty \grave{ 
 \gamma}^t \grave G_t  \frac{\partial \ln\left ( \grave{\pi}_i\left ( (\grave S_t,\grave{U}_t^\text{pre}), \grave{U}_t, \theta \right ) \right )}{\partial \theta} \Big | \theta \Big].
$$

Consider $\partial \ln \left ( \grave{\pi}_i ( (\grave S_t,\grave{U}_t^\text{pre}), \grave{U}_t, \theta  ) \right ) / \partial \theta$, which we abbreviate as $\partial \grave{\pi}_i / \partial \theta$.
When $\grave{U}^{\pre}_t.e_i = 0$, we know that the action is $\grave{U}^i_{t} = \grave{S}_t.u_i = \grave{U}^i_{t-1}$ regardless of $\theta$.
Therefore, in these local states, $\partial \grave{\pi}_i / \partial \theta$ is zero.
When $\grave{U}^{\pre}_t.e_i = 1$, we see from the definition of $\grave{\pi}$ that $\partial \grave{\pi}_i / \partial \theta {=} \partial \pi_i / \partial \theta$.
Therefore, we see that in all cases, $\partial \grave{\pi}_i / \partial \theta {=} (\grave{U}^{\pre}_t.e_i) \partial \pi / \partial \theta$.
Substituting this into the above expression yields:
\begin{align*}
    \mathbb{E} \Big[& \sum_{t=0}^\infty (\grave{U}^{\pre}.e_i) \grave{\gamma}^t  \grave{G}_t \cdot \\
    &\frac{\partial \ln\left ( \grave{\pi}_i\left ( (\grave S_t.s,\grave{U}_t^\text{pre},\grave S_t.u^{\text{all}}), \grave{U}_t, \theta \right ) \right )}{\partial \theta} \Big | \theta \Big].
\end{align*}
In the proof that $J(\theta) = \grave J(\theta)$ given in Section C of the supplementary material, we show that the distribution over all analogous random variables is equivalent in both settings (for example, for all $s \in \mathcal S$, $\Pr(S_t=s) = \Pr(\grave S_t.s = s)$).
Substituting each of the random variables of $M$ into the above expression yields precisely the asynchronous local policy gradient, $\Delta_i$.''
\end{quote}

\begin{algorithm}
\caption{Update and helper functions for Algorithm \ref{TabularAlg}}\label{UpdateAndHelperFunctionsTabular}
\begin{algorithmic}[1]
\Procedure{PrimitiveAction}{$s, t, o$}
\Repeat
\State $o.ActivationTime \gets t, \ \ o.Reward \gets 0$
\State $o' \gets \pi_{o}(s)$ \indent \indent \indent // $o' \in o.Children$ 
\State $PathToRoot.append(o')$
\State $o.PrevAction = o', \ \ o.LastObservation = s$
\State $o'.ActiveParent = o$
\State $o\gets o'$
\Until{$o$ is primitive}
\State \textbf{return} $o, PathToRoot$
\EndProcedure
\Procedure{UpdateCritics}{$s, PathToRoot, t, \text{done}$} \label{updatecriticsalgo}
\State // To compute $v_{o}$ at $s$, refer to \ref{parentfortargetcomputation}
\For{$o \in PathToRoot.Reverse$} 
\State // Starting from $o =$ root, where $\beta_o \equiv 0$
\State $v_{o} \gets (1-\beta_o(s))Q_{\pi_{o.ActiveParent}}(s,o)+\beta_o v_{o.ActiveParent}$ \label{UpdateCriticsParentTarget}
\EndFor
\For{\texttt{$o \in PathToRoot$}}
\State $x_o, u_o \gets o.LastObservation, o.PrevAction$
\State $\delta_o \gets Q_{\pi_o}(x_o,u_o) - (o.Reward + \gamma^{t-o.ActivationTime}v_{o})$
\State $Q_{\pi_o}(x_o,u_o) \gets Q_{\pi_o}(x_o,u_o)+\alpha_Q\delta_o$
\If{not $o.Terminated$}
\State \textbf{break} 
\EndIf
\EndFor
\EndProcedure
\Procedure{UpdateActors}{$PathToRoot, t, \text{done}$}
\For{\texttt{$o \in PathToRoot$}}
\State $c \gets o.PrevAction$
\State $x_o \gets o.LastObservation$
\State // Using active parent action value instead of $\max_{o'} Q_{\pi_o}(x_o,o')$ for baseline
\State $\theta_{\pi_o} \gets \theta_{\pi_o} + \alpha_\pi\frac{d\pi_o}{d\theta_{\pi_o}}(x_o,c)(Q_{\pi_o}(x_o,c) - Q_{\pi_{o.ActiveParent}}(x_o,o))$ 
\If{not $o.Terminated$}
\State \textbf{break}
\EndIf
\EndFor
\EndProcedure
\Procedure{UpdateBetas}{$s,PathToRoot,t,\omega$}
\State // Not the same as in line \ref{UpdateCriticsParentTarget}, only computing $v_\omega$ 
\For{$o \in PathToRoot.Reverse$}
\State $v_{o} \gets (1-\beta_o(s))\max_{o'}Q_{\pi_{o}}(s,o')+\beta_o v_{o.ActiveParent}$ \label{UpdateTerminationParentTarget}
\If{$o == \omega$}
\State \textbf{break}
\EndIf
\EndFor
\State\label{vomegauseline} $q \gets \max_{o'}Q_{\pi_\omega}(s,o')$   // used in line \ref{vomegauseline2}
\State $Probs \gets 1$ \label{impsamplingbeta}
\For{\texttt{$o \in PathToRoot$}}
\If{$o.Terminated$}
    \State \label{vomegauseline2}  $q_{\beta_o} \gets q$  // see \ref{vomegause} 
    \State $\theta_{\beta_o} \gets \theta_{\beta_o} + \alpha_\beta  \cdot Probs \cdot \frac{d\beta_o}{d\theta_{\beta_o}}(s)(q_{\beta_o}-v_{o})$
    \State $Probs \gets Probs \cdot \beta_o(s)$
\Else  // $o=\omega$ and the opposite of the gradient $\beta_o$ has to be used (see \ref{wrongbetaupdate})
    \State $q_{\beta_o} \gets \max_{o'} Q_{\pi_{o}}(s,o')$
    \State $\theta_{\beta_o} \gets \theta_{\beta_o} - \alpha_\beta  \cdot Probs  \cdot \frac{d\beta_o}{d\theta_{\beta_o}}(s)(q_{\beta_o}-v_{o})$
    \State \textbf{break} 
    \State // No update for options that have not been called back
\EndIf
\EndFor 
\EndProcedure
\end{algorithmic}
\end{algorithm}

\section{Inherent learning Stabilization in Hierarchical Options}\label{FON_app}
We want to understand the algorithm for $\left < 1,1 \right >$ by building it piece by piece starting from an AC algorithm and changing it where appropriate. For the sake of illustration, we will assume that we have a fixed value of termination $\frac{1}{n}$ for some $n>0$. This means the option terminates every $n$ steps. This is also somewhat of a reasonable assumption for $\langle 1 , 1\rangle$, as observed in some of the FR experiments, where unless termination temperature is low, the termination function (more or less) stays constant throughout the training of $\left < 1,1 \right >$ (\cref{12and122optionlength,1and11and111optionlength}).

First, note that the root has one action (selecting the child), and therefore, its action value is the value function of the children. Therefore, let us take a vanilla AC algorithm and add a value target network computing $V_{target}$ for the algorithm, where:
\begin{enumerate}
    \item It is only updated every $n$ ($=(\frac{1}{n})^{-1}$) steps using its own rollout update rule 
    \begin{align*}
        V_{target}(s)\leftarrow& r_1+\gamma r_2+\ldots +\\
    &\gamma^{n-1}r_{n-1} + \gamma^n V_{target}(s^{(n)}).
    \end{align*}
    \item The target is used to soft-update the $Q$-value of the AC, where instead of choosing $V_{target}(s')$ to build the target of $Q(s,a)$, we choose $\frac{1}{n}V_{target}(s')+ (1-\frac{1}{n})\max_a Q(s',a)$, meaning
        $$Q(s,a) \leftarrow r + \gamma( \frac{1}{n}V_{target}(s')+ (1-\frac{1}{n})\max_a Q(s',a)),$$
    instead of  
        $$Q(s,a) \leftarrow r + \gamma V_{target}(s').$$
\end{enumerate}    
Notice there are two effects: (1) the rollout effect that trickles down from (a) to the update rule (b), (2) the target network used softly in the update, with the rate being the same as its update frequency, i.e. $1/n$.

The above assumed a fixed termination function. In general, the termination function also learns, so the \textit{rate of the target network update} is learned. In summary, $\left < 1,1 \right >$ can be viewed as an actor-critic model which is soft-updated with a value target network with rollout and a learnable rate of update. This is the reason behind $\left < 1,1 \right >$ outperforming the vanilla AC model as it is equipped with several improvements. This ``target effect'' is generalized/compounded if one uses more levels of hierarchy, i.e. $\left< 1, 1, \ldots ,1 \right >$. This also explains why in other hierarchical networks, such as HAC, it is unnecessary to use a target network. 

\section{Experiment Settings}\label{appsec:experiments}
At each episode of the Four-Rooms task, the agent starts at a random cell and the goal state to be reached is also picked uniformly from the 104 cells, which are divided into four square rooms dividing a larger square, with four one-cell corridors connecting the adjacent rooms. Each action succeeds with a probability of $\frac{2}{3}$, and if failed, the agent goes randomly to one of the other adjacent cells. The reward is $1$ if the goal is reached, and zero otherwise.
In all models tested below, 
all coagents $Q$-values are initialized at zero, with policies being uniform distributions, e.g. the termination probabilities are $\frac{1}{2}$. Unless mentioned otherwise, we take 50 runs with five to ten different random seeds for 50000 episodes, and show the 500 moving average of the quantity of interest, which is either the number of steps to the goal or the option lengths.

\textbf{No target network used. }As mentioned previously and theoretically justified in \cref{FON_app}, each parent is effectively acting as a value target network with rollout. 

\textbf{Hyperparameters. }The hyperparameters in all experiments are the same, unless mentioned explicitly otherwise. We choose $\gamma = 0.99$, termination and actor temperature $1$ and  $0.01$, and termination, critic and actor learning rates $\alpha_\beta = 0.001, \alpha_Q = 0.01$ and $\alpha_\pi = 0.00001$, respectively. More details can be found in the code.

\section{Experiments on Performance}\label{sec:exp_performance}
\textbf{Analysis and comparison of results. }Our OC/HOC models generally outperform their previous versions in the literature as shown in \cref{treesvsmatt}. The experiments also demonstrate that increasing the number of lowest options is not always beneficial (\cref{100000epsrun}). 
This finding is confirmed by \cite{riemer2018learning}, but as further illustrated in \cref{100000epsrun} for HOC and FON models, more options simply means slower learning for this task, and not a reduction in the best possible performance. However, taking into account the efficiency of our algorithm as discussed in \cref{sec:algorithm}, the training time for HOC with $4$ levels to reach score below 150 in 100000 episodes is less than that of \cite{riemer2018learning} (reached in 50000 episodes).

\begin{figure}[h!]
    \centering
    \includegraphics[width = \columnwidth]{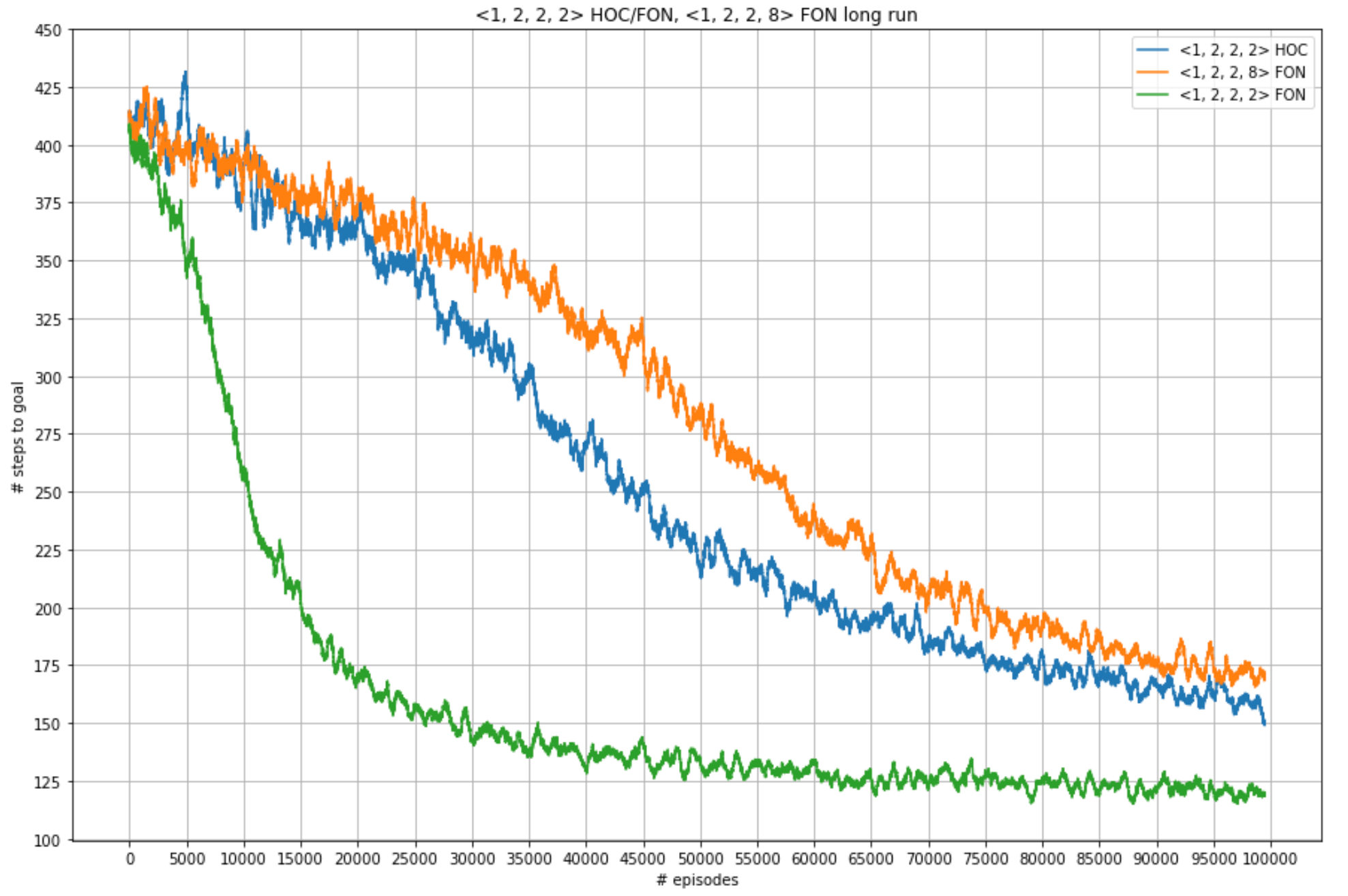}
    \caption{The results for the HOC model $\left<1, 2, 2, 2\right>$ and FON models $\left<1, 2, 2, 8\right>$ and $\left<1, 2, 2, 2\right>$ runs for 100000 episodes for 5 random seeds. We see the eventual convergence of the HOC to the performance achieved by the same architecture in \cite[Fig. 3]{riemer2018learning} for HOC ($N=4$). We hypothesize the slower convergence of the HOC model is mostly due to the number of lowest options present in this model. This is demonstrated by the similar performance of the FON model $\left<1, 2, 2, 8\right>$ with the same number of lowest options, versus the better performing $\left<1,2,2,2\right>$ FON model with the same number of levels of hierarchy but only two lowest options.}
    \label{100000epsrun}
\end{figure}

\textbf{Comparison of FON to AC. }We mentioned in \cref{FON} (with more details in \cref{FON_app}) that the FON model $\left < 1,1 \right >$ has some improvements over a vanilla Actor-Critic model which is represented by $\left < 1 \right >$. We see this fact reflected in the performance of $\left < 1,1 \right >$ in \cref{fonruns}, compared to all of the other models, even the deeper ones.

\begin{figure}[h!]
    \centering
    \includegraphics[width = \columnwidth]{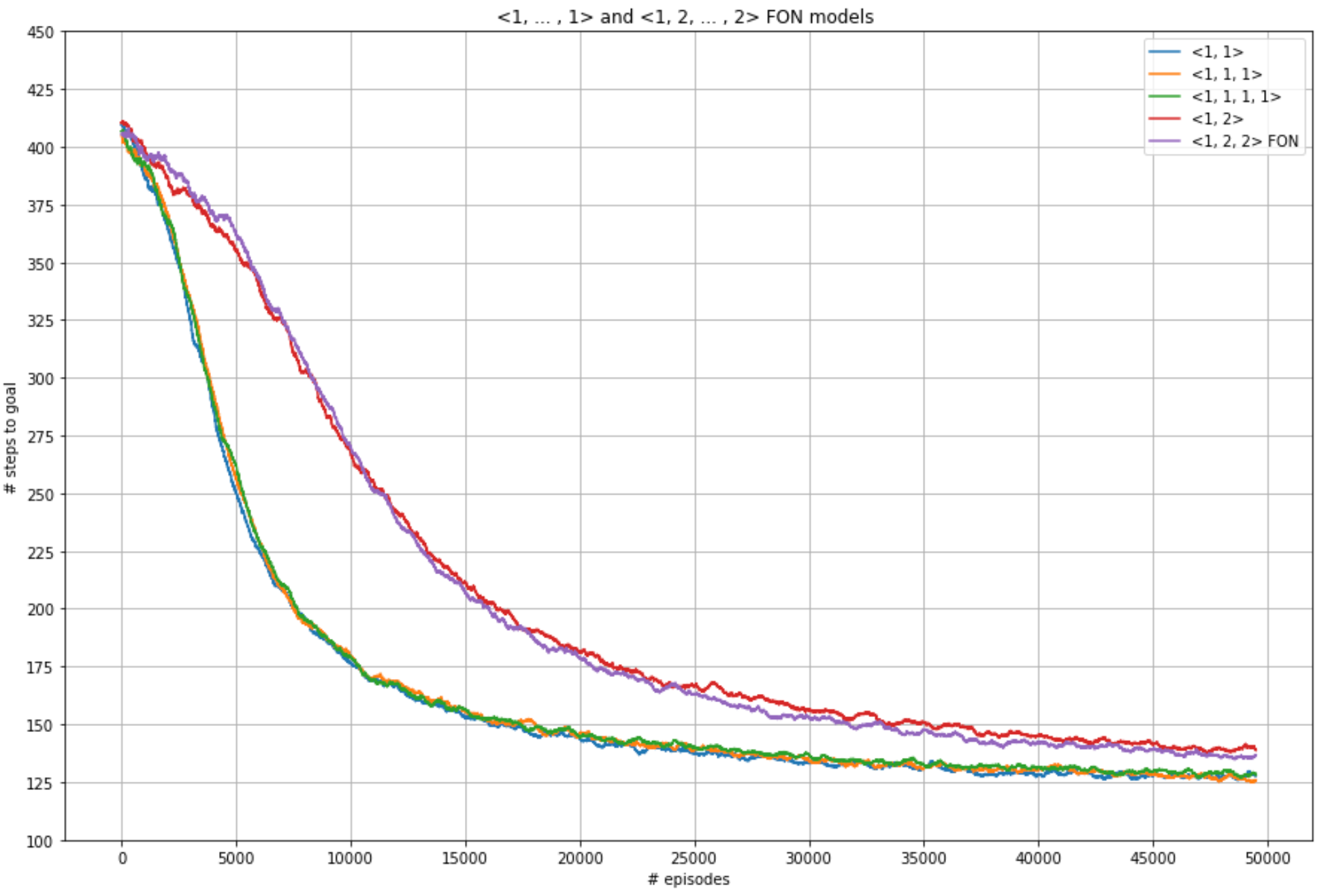}
    \caption{We compare the results of some simple FON models. The best model, in terms of architectural simplicity, convergence time and eventual performance is the simplest model $\left<1, 1\right>$. This also outperforms the previous HOC models in \cref{treesvsmatt}. Note that only two options, a root and a child are present in this model, and the primitive actions are executed by the child. During learning, the presence of another option (the root) is essential. This is further discussed in \ref{FON_app} and also supported by the result of the AC model in \cite[Fig. 3]{riemer2018learning}.}
    \label{fonruns}
\end{figure}

\section{Experiments on Option Properties}\label{sec:optlengthandtermtemp}
A main issue in option models has been the length of the options, with the model converging to one option dominating, or most options being of very small length (high termination probability). In the incorrect update model of $\beta$ (as discussed in \cref{wrongbetaupdate}), a proposal was to introduce a deliberation cost $\eta$ to the advantage function in the update rule.

\textbf{Tuning the termination temperature. }Here, we propose a simple method which does not require any additional hyperparameter or architectural modification. We examine the effect of lowering the temperature of the sigmoid function used in computing $\beta$. The idea is that as action values are initialized with $0$s, the termination probabilities initially drop. Then, due to the numerics of a sigmoid function with low temperature, it should become numerically harder for said probabilities to rise, thus contributing to lengthier options.
However, this reduces the update frequency of higher options. Over a long period of time, one wonders if the higher options can learn their value function as well, which requires experimentation.

\textbf{$\left < 1, \ldots, 1\right>$ and $\left < 1, 2, \ldots, 2\right>$ FON models with low temperature. }The models we experiment with have termination temperature 0.05, compared to the normal 1.0 in the previous section. Option lengths dramatically increase (\cref{12and122optionlength,1and11and111optionlength}), spanning a significant part of the episode length while leaving room for other options to be used/called back as well. So no option in the $\left < 1, 2, \ldots, 2\right>$ FON models becomes dominant, as also shown by their almost equal length, meaning both lowest options are active participants in the task.  \cref{12slowandnormaltemp} shows that lowering the temperature slightly degrades the performance in 50000 episodes but longer runs may close this gap.
\begin{figure}[h!]
    \centering
    \includegraphics[width = \columnwidth]{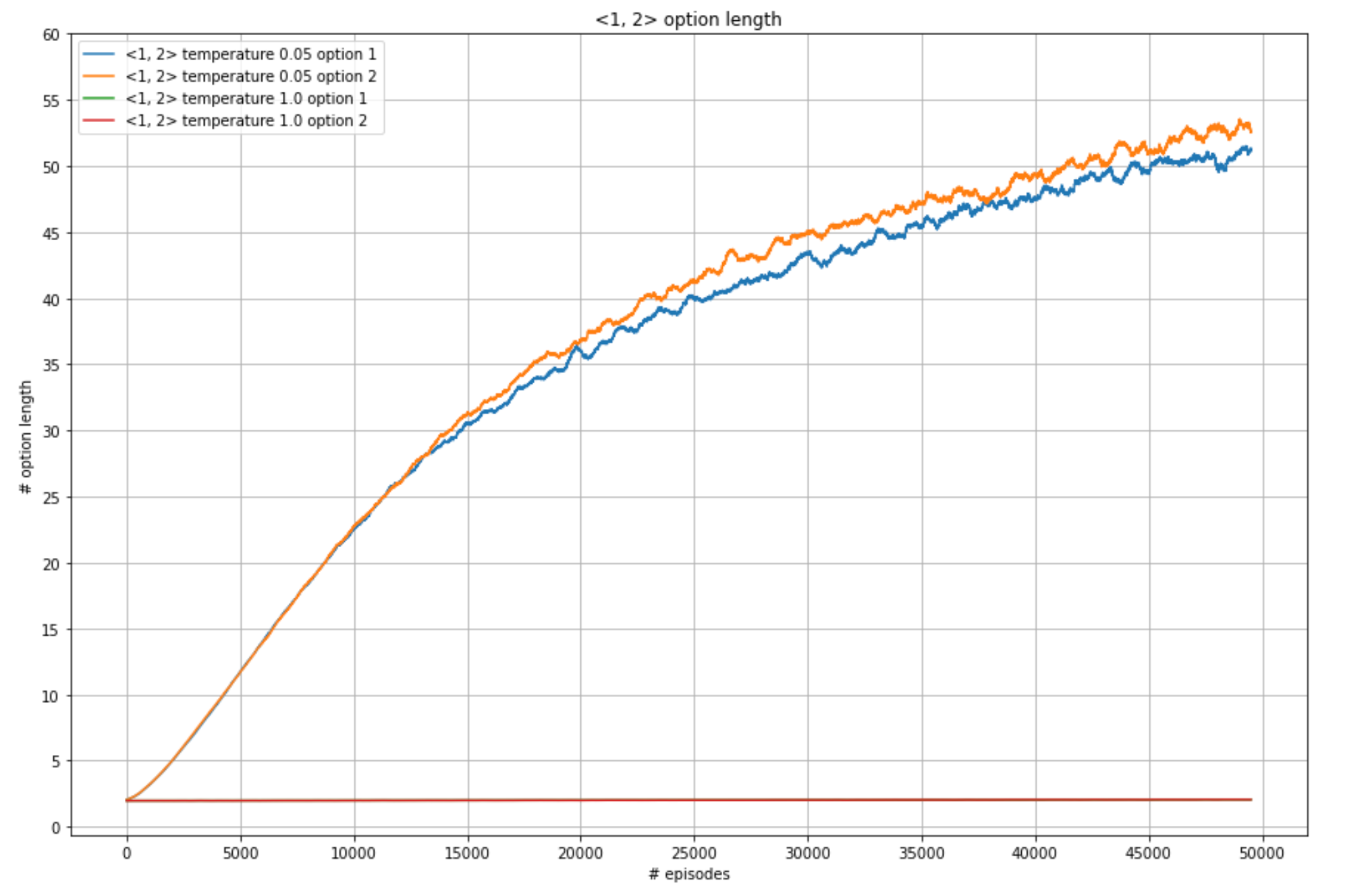}
    \includegraphics[width = \columnwidth]{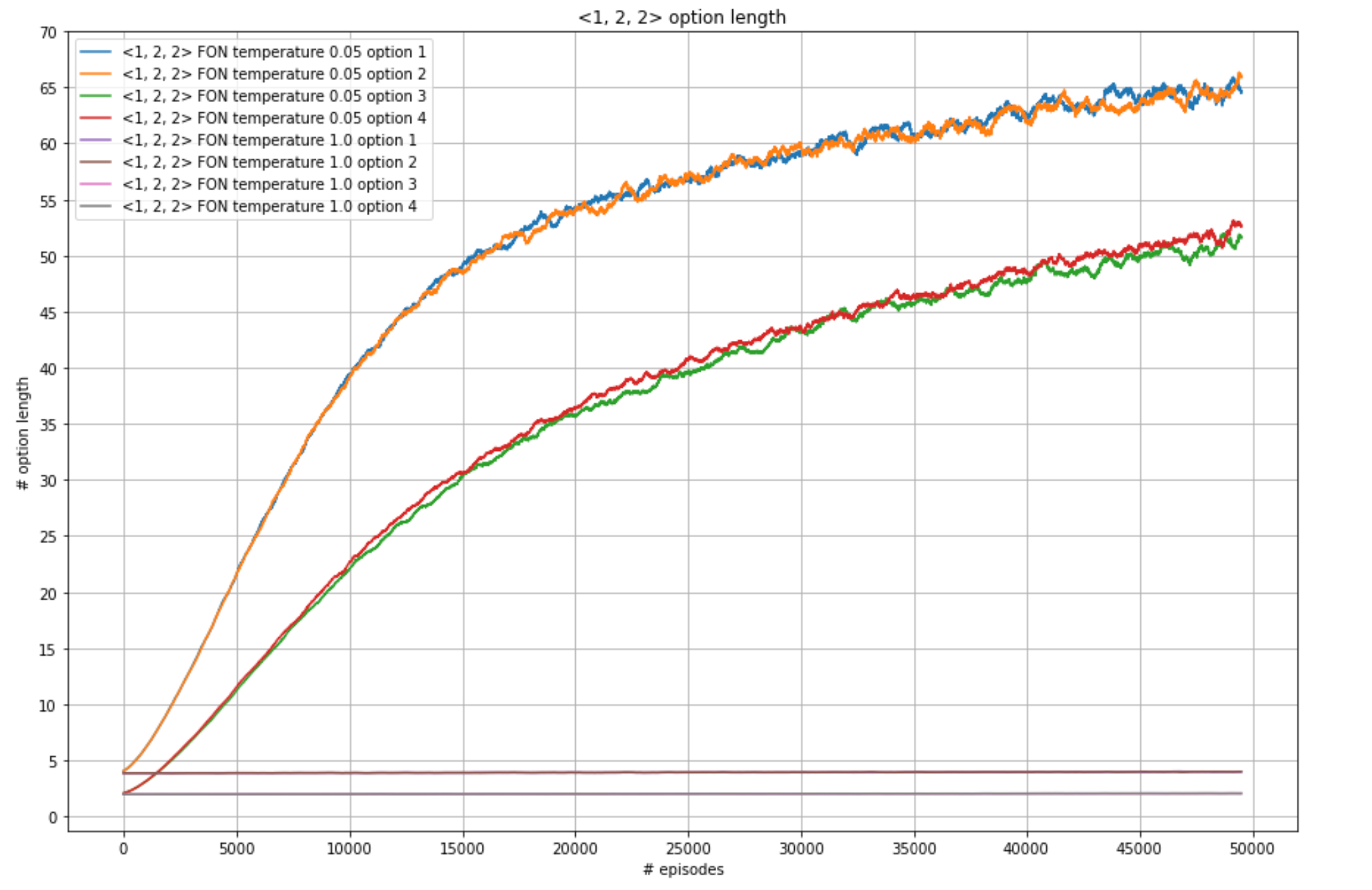}
    \caption{Option lengths compared for the $\left<1,2, \ldots, 2\right>$ models with low and normal temperature. Options are labelled by $\{0,\ldots,O\}$ where $0$ denotes the root, $1,2$ the second and third options at the first level, ... , and $O-1,O$ the lowest two options. No option becomes dominant as their length is almost equal and around half of the episode length.}
    \label{12and122optionlength}
\end{figure}


\begin{figure*}[h]
    \centering
    \includegraphics[width = \columnwidth]{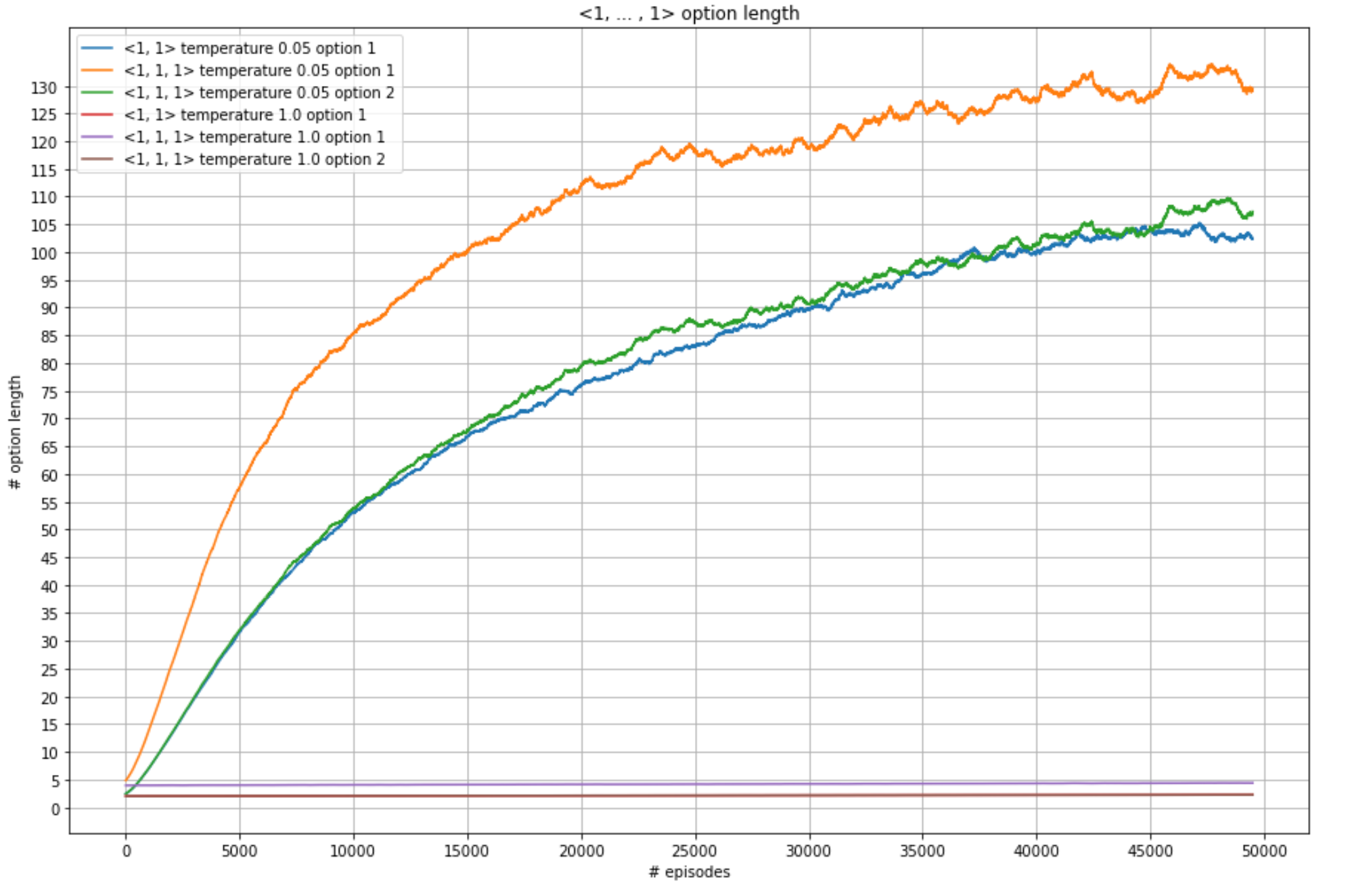}
    \includegraphics[width = \columnwidth]{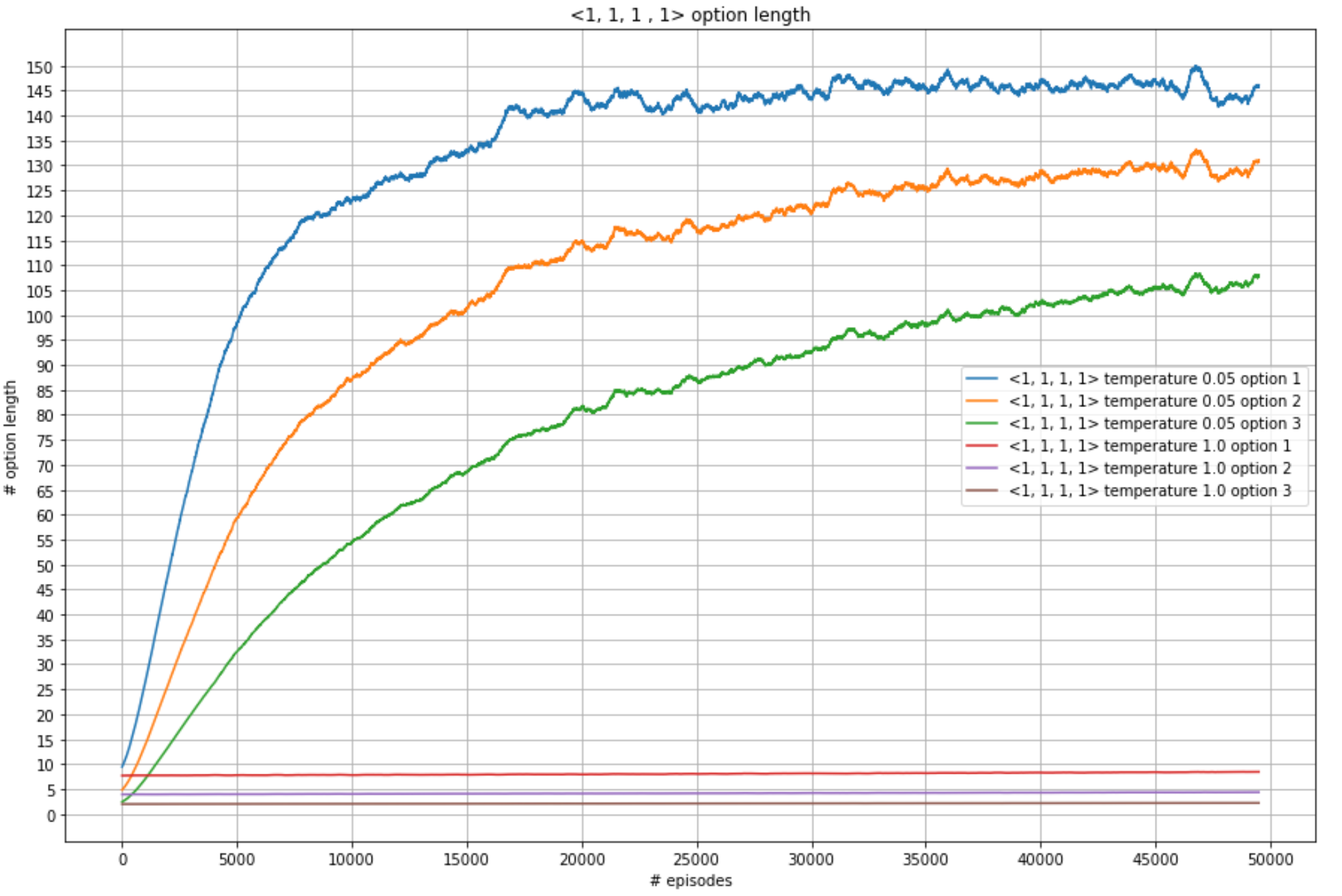}
    \caption{Option lengths compared for the $\left<1, \ldots, 1\right>$ models with low and normal temperature. Options are enumerated by $\{0,\ldots,O\}$ where $0$ denotes the root and $O$ the lowest option. Normal temperature leads to almost constant (very small) length throughout the entire training, while low temperature leads to a stable increase and eventual stabilization. Even though the option length for intermediary options has almost attained the episode length, they and the root are called back (used at least twice in an episode), as their child option length is smaller than the episode length (above 150 as shown in \cref{12slowandnormaltemp}).}
    \label{1and11and111optionlength}
\end{figure*}
\begin{figure*}[h]
    \centering
    \includegraphics[width = \columnwidth]{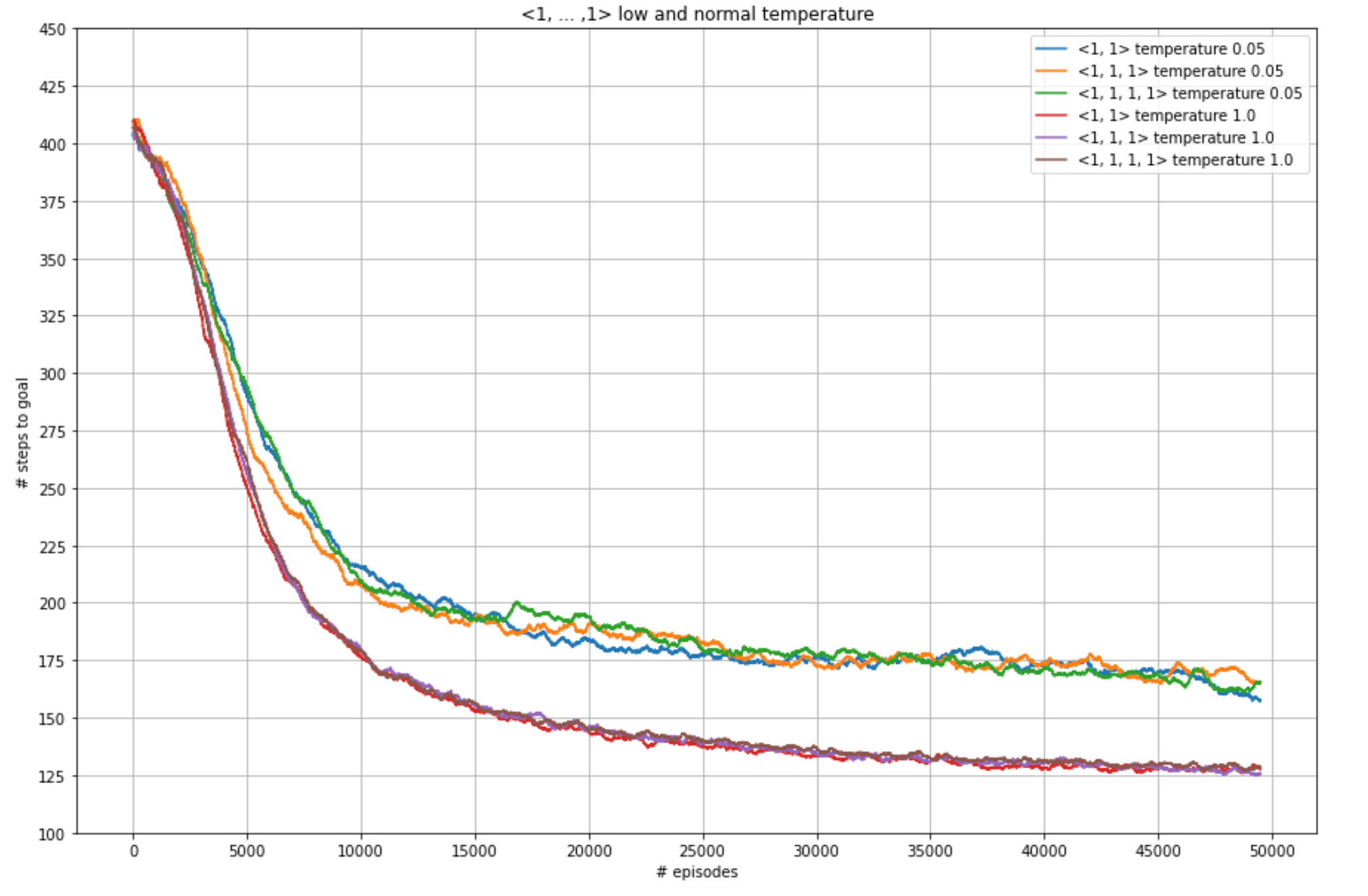}
    \includegraphics[width = \columnwidth]{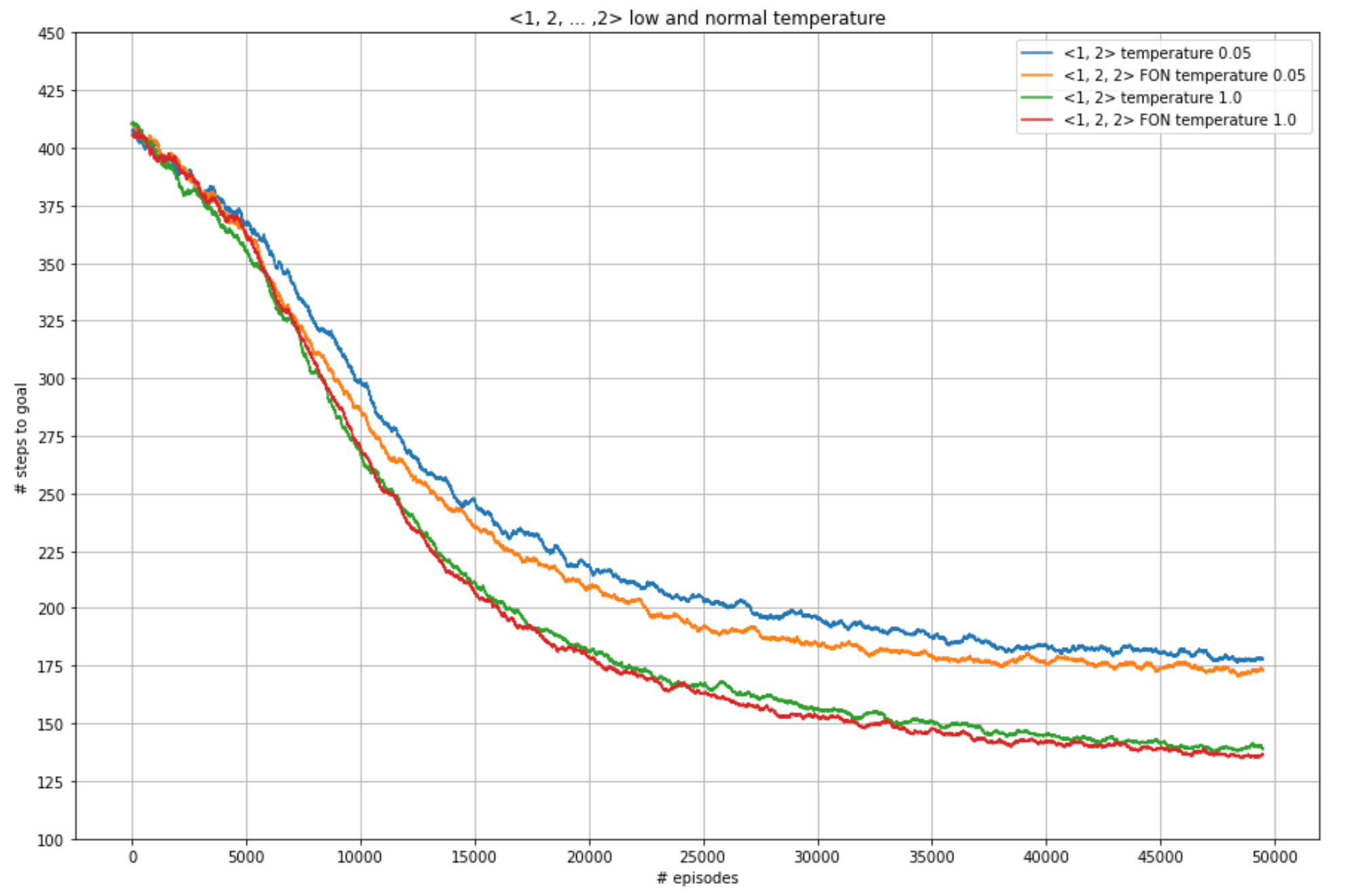}
    \caption{Lower temperature seems to slow the convergence to the best possible performance in both $\left < 1, \ldots, 1\right>$ and $\left < 1, 2, \ldots, 2\right>$ FON models.}
    \label{12slowandnormaltemp}
\end{figure*}

\section{More on Experiments}\label{appsec:more_on_exp}

\textbf{Stability and option lengths. }\label{longrunsexps}
The instability of RL algorithms is well-known, especially in nonstationary settings. We run the experiments for ten times longer and with up to seven levels of hierarchy, and observe in \cref{levlongrunstability,levlongrunoptlength} that increasing the width correlates with stability while preserving the lowest level options' lengths. Also, instability in those models correlate with spikes in said lengths. We hypothesize that as the reward function changes dramatically from episode to episode (due to the goal changing), it can be easier for the model to be stable in the long term if there is more width to the network.

To test the stability, we take runs of 500000 episodes on a single random seed, for FON models of the form $\left<1,\ldots,1\right>$ and $\left<1,2,\ldots,2\right>$ in \cref{levlongrunstability}. We observe how the stability of the algorithm in the long term is generally better for the $\left<1,2,\ldots,2\right>$ models. As mentioned, we can see how the crashes in performance, which are spikes in the graphs, are correlated with those present in the option length of the lowest option(s) as shown in \cref{levlongrunoptlength}. Note that none of the lowest options are dominant and they are also lengthy in their uninterrupted execution, confirming \cref{12and122optionlength} for deeper models.

\textbf{Estimating $Q_{\beta_o}$ using the not-terminated ancestor. }\label{vomegause}Similar to how parents were used as target networks in \cref{sec:algorithm} for the policies $\pi_o$, the same approach applies for termination functions as $Q_{\beta_o}(s,\text{termination})= V_{\beta_{\text{parent}(o)}}(s)$. Considering that $\text{parent}(o)$ may have also terminated in the execution path, one could the same value for grandparents of $o$, and so on, up to the ancestor that has not terminated. Based on our experiments, this last option is a better estimate of the true $Q_{\beta_o}(s,\text{termination})$, as that is the beginning of the next phase of execution and thus, is a closer estimate to the next collected rewards (line \ref{vomegauseline} of Algorithm \ref{UpdateAndHelperFunctionsTabular}).



\begin{figure*}[h]
    \centering
    \includegraphics[width = \columnwidth]{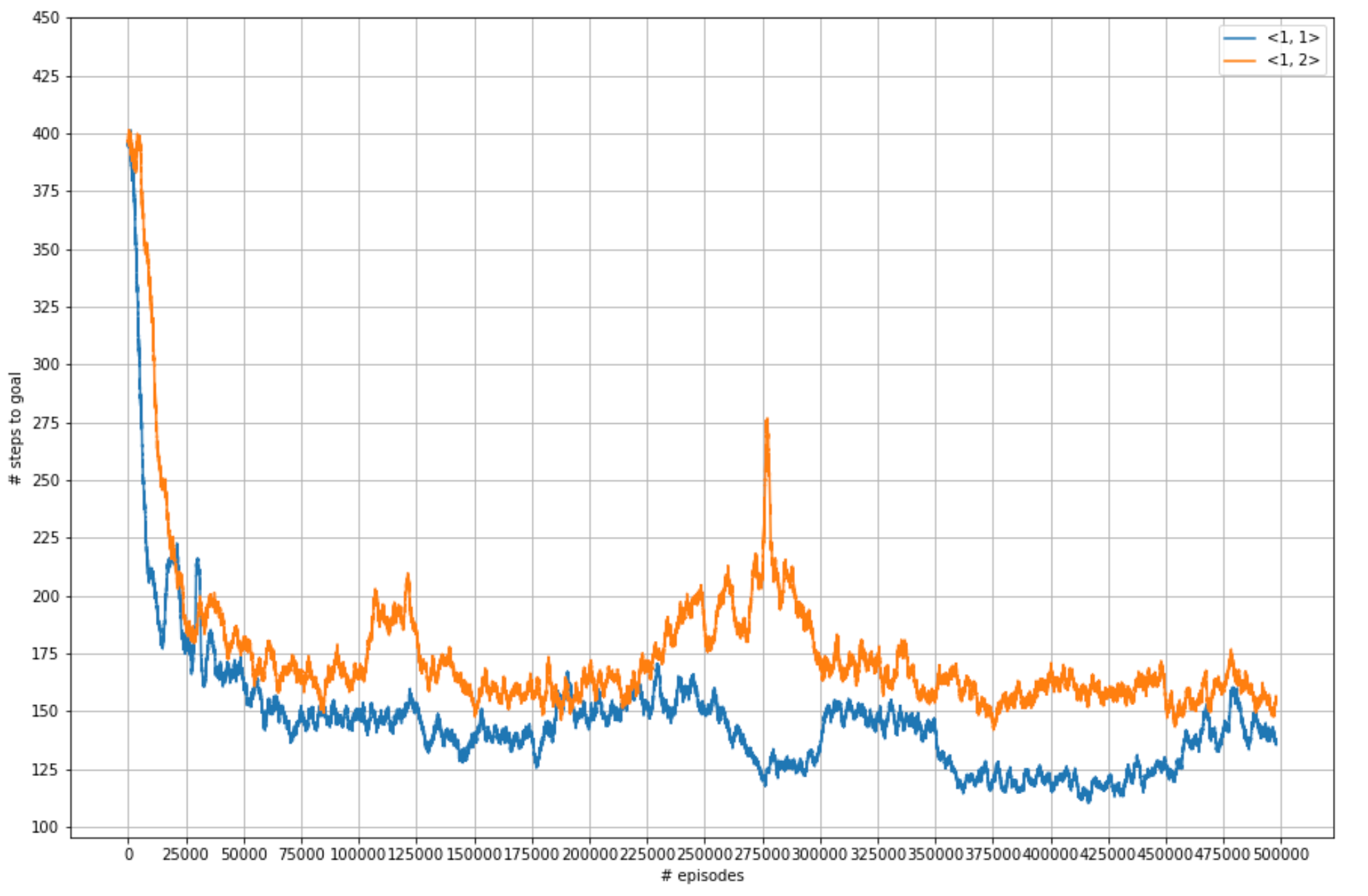}
    \includegraphics[width = \columnwidth]{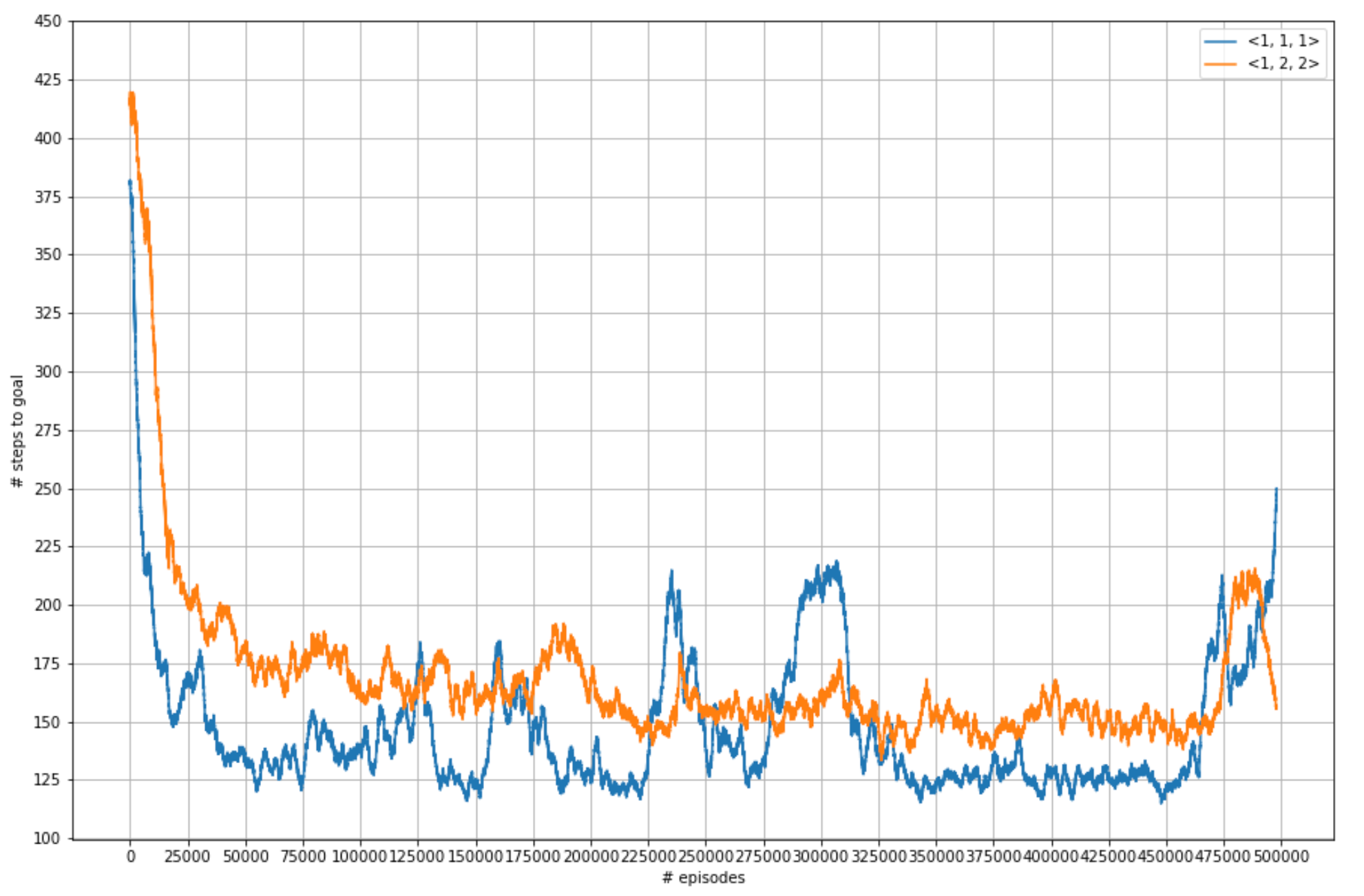}
    \includegraphics[width = \columnwidth]{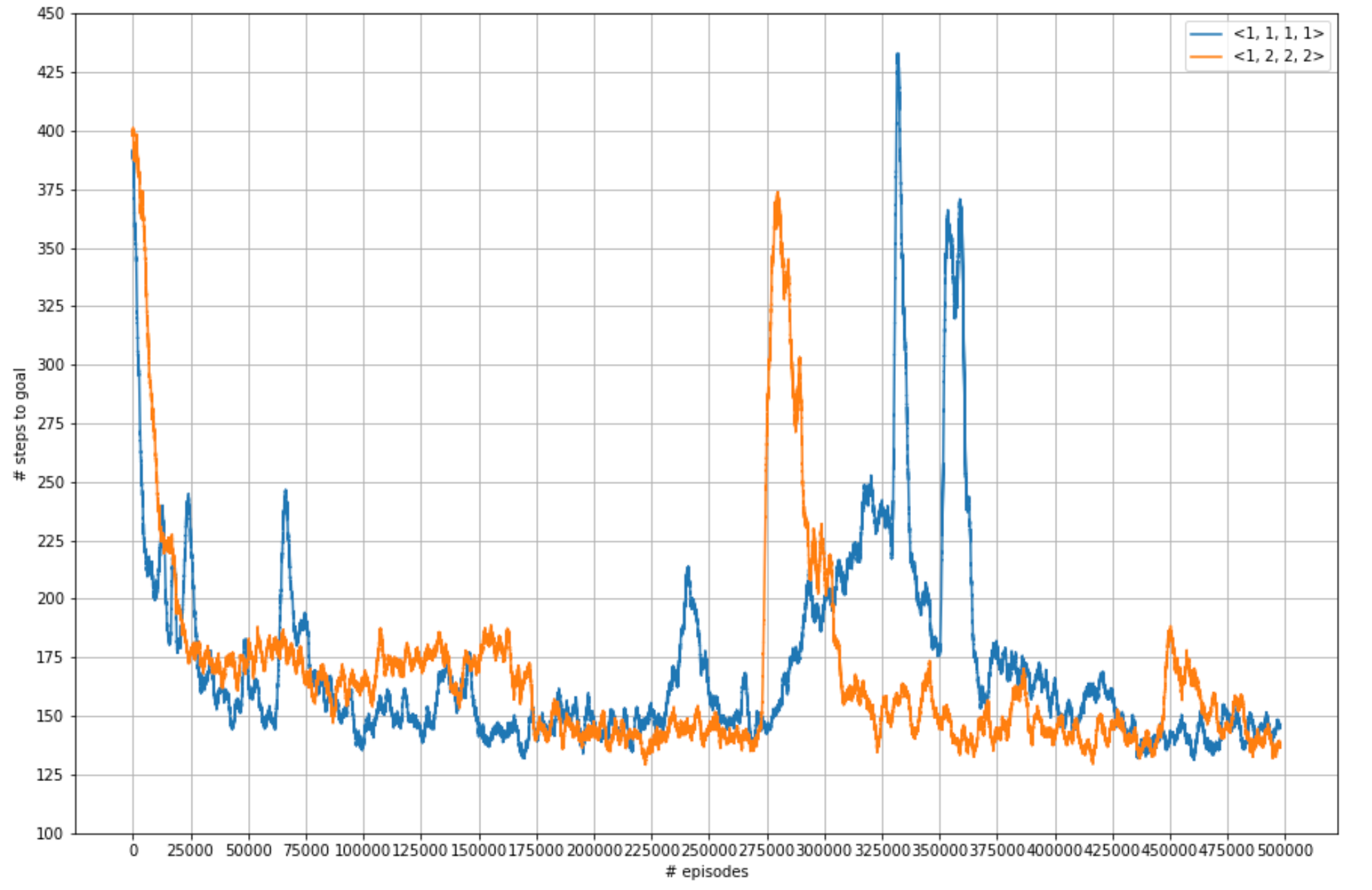}
    \includegraphics[width = \columnwidth]{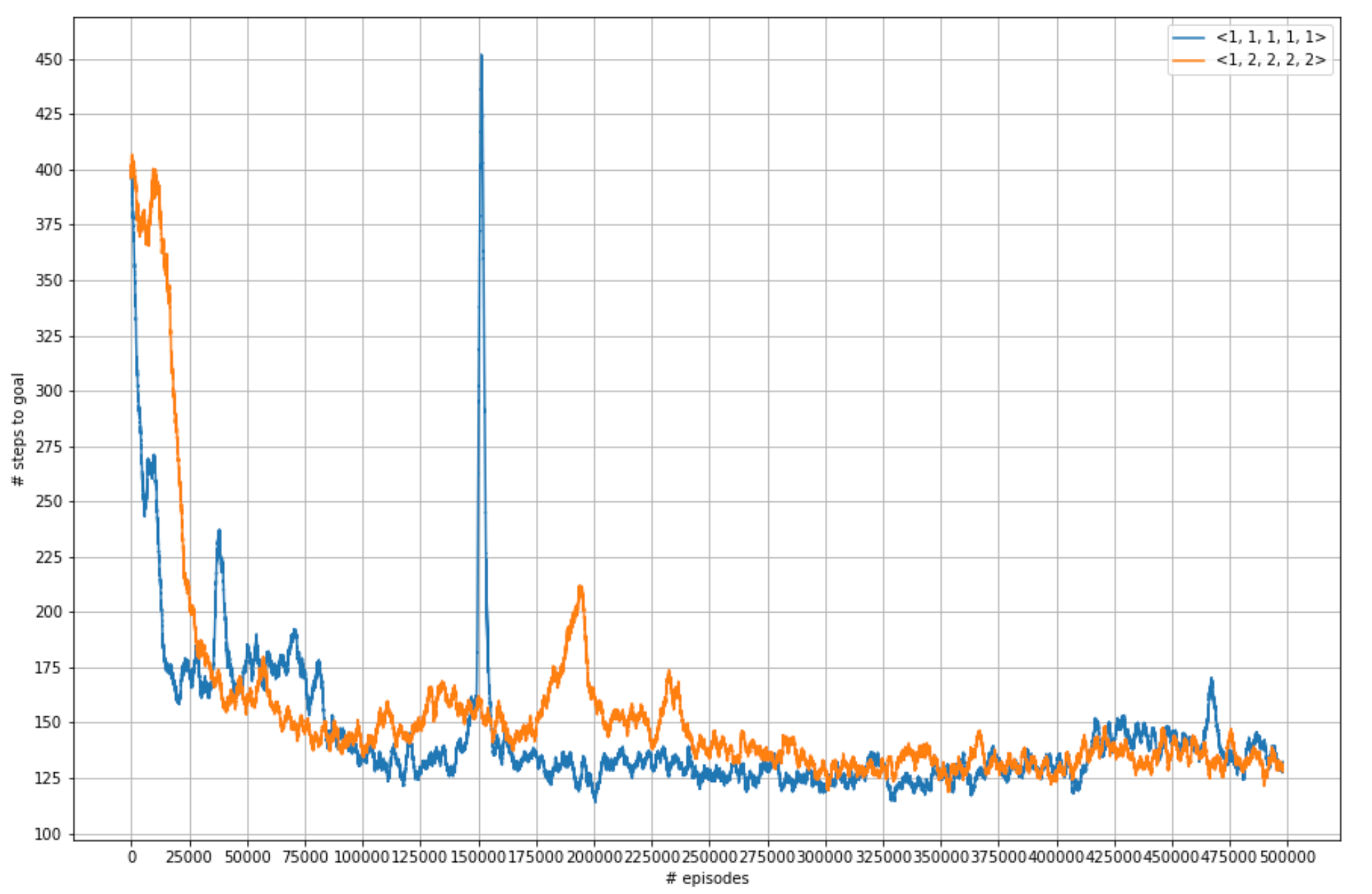}
    \includegraphics[width = \columnwidth]{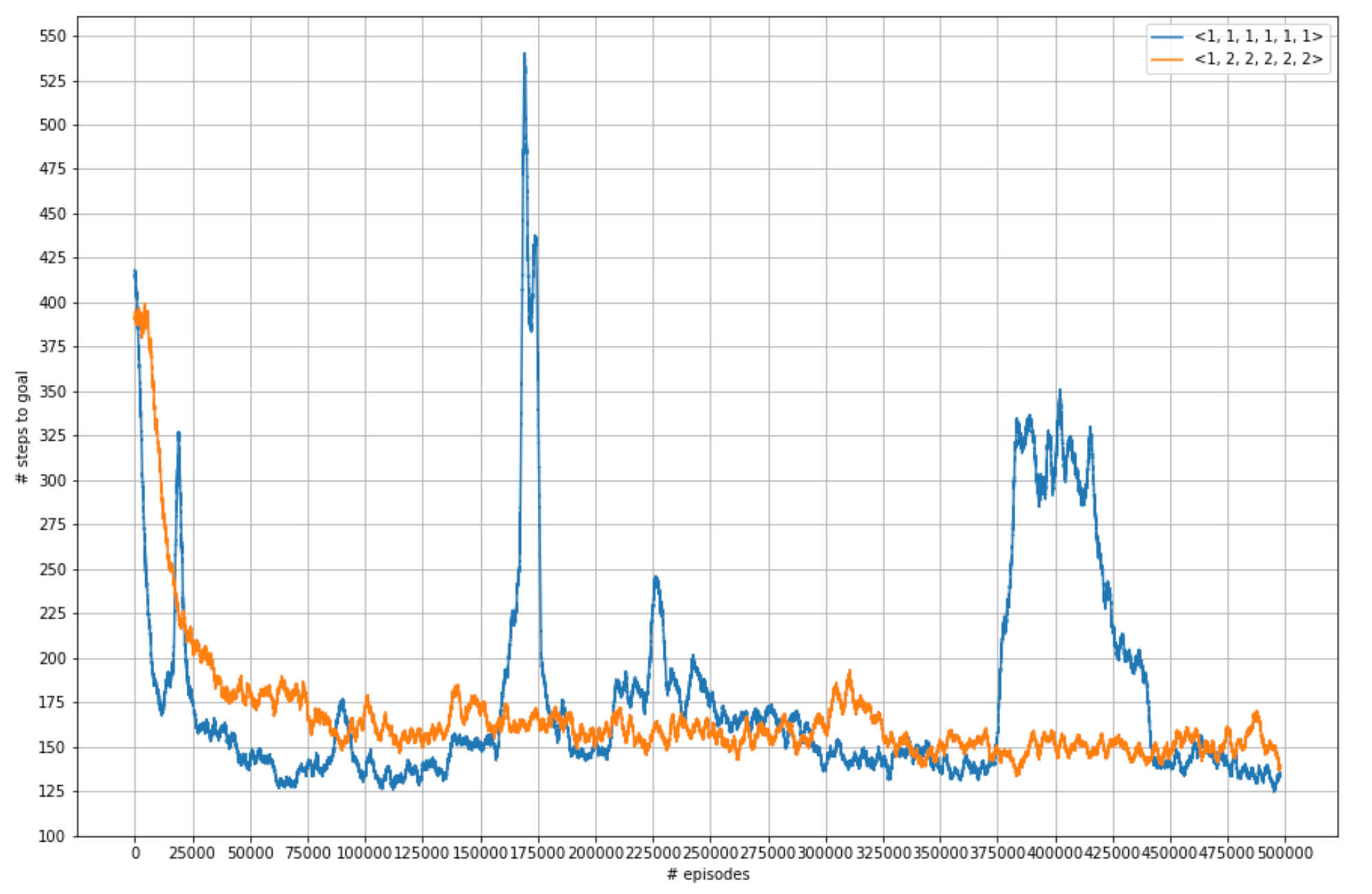}
    \includegraphics[width = \columnwidth]{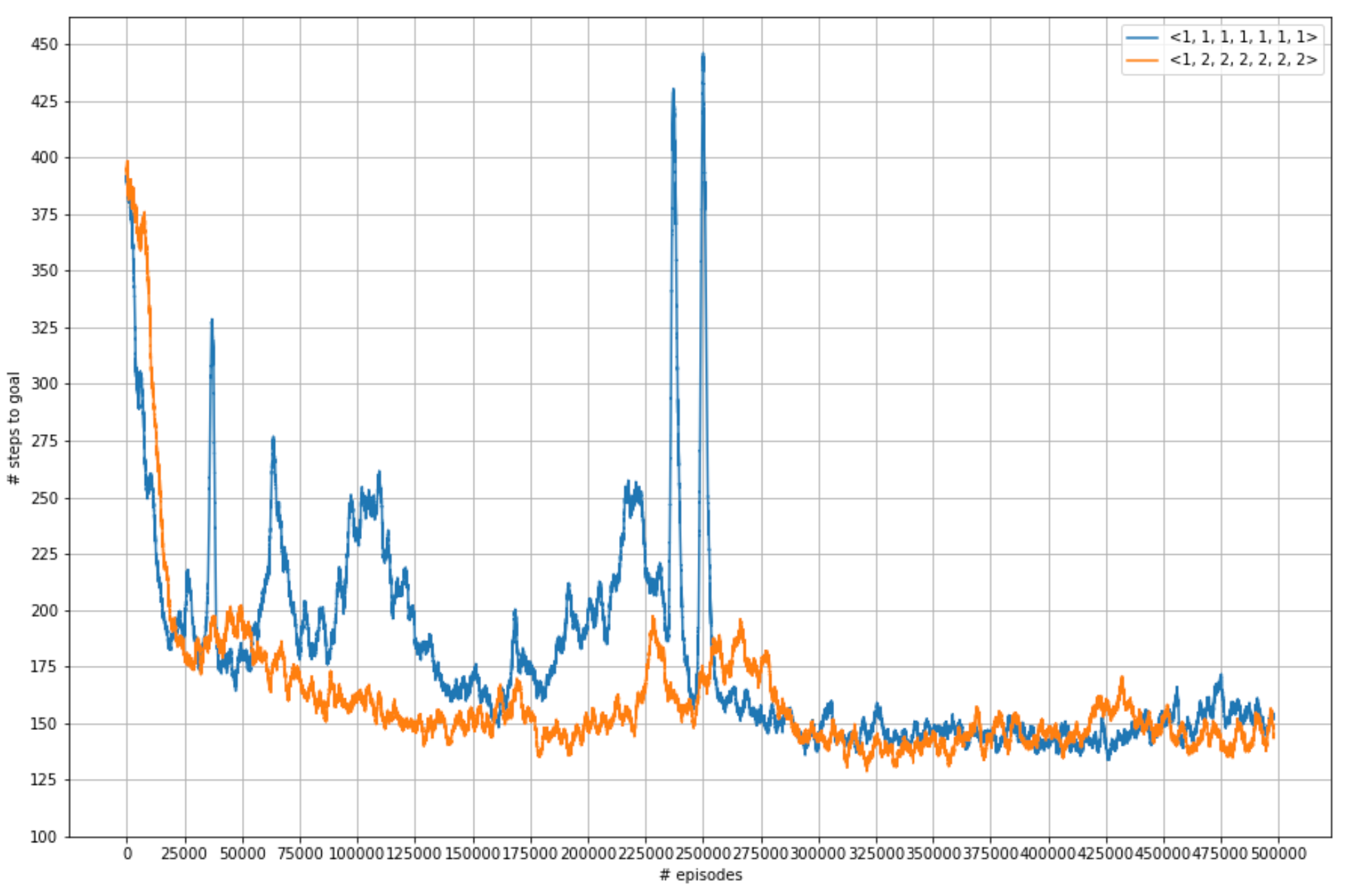}
    \caption{Performance comparison of $\left<1,\ldots,1\right>$ and $\left<1,2,\ldots,2\right>$ FON models with the same number of levels of hierarchy (two to seven). General stability increases as we switch from the former to the latter model. Note that almost all deeper models achieve the best performance of 125 steps as in \cref{fonruns}.}
    \label{levlongrunstability}
\end{figure*}

\begin{figure*}[h]
    \centering
    \includegraphics[width = \columnwidth]{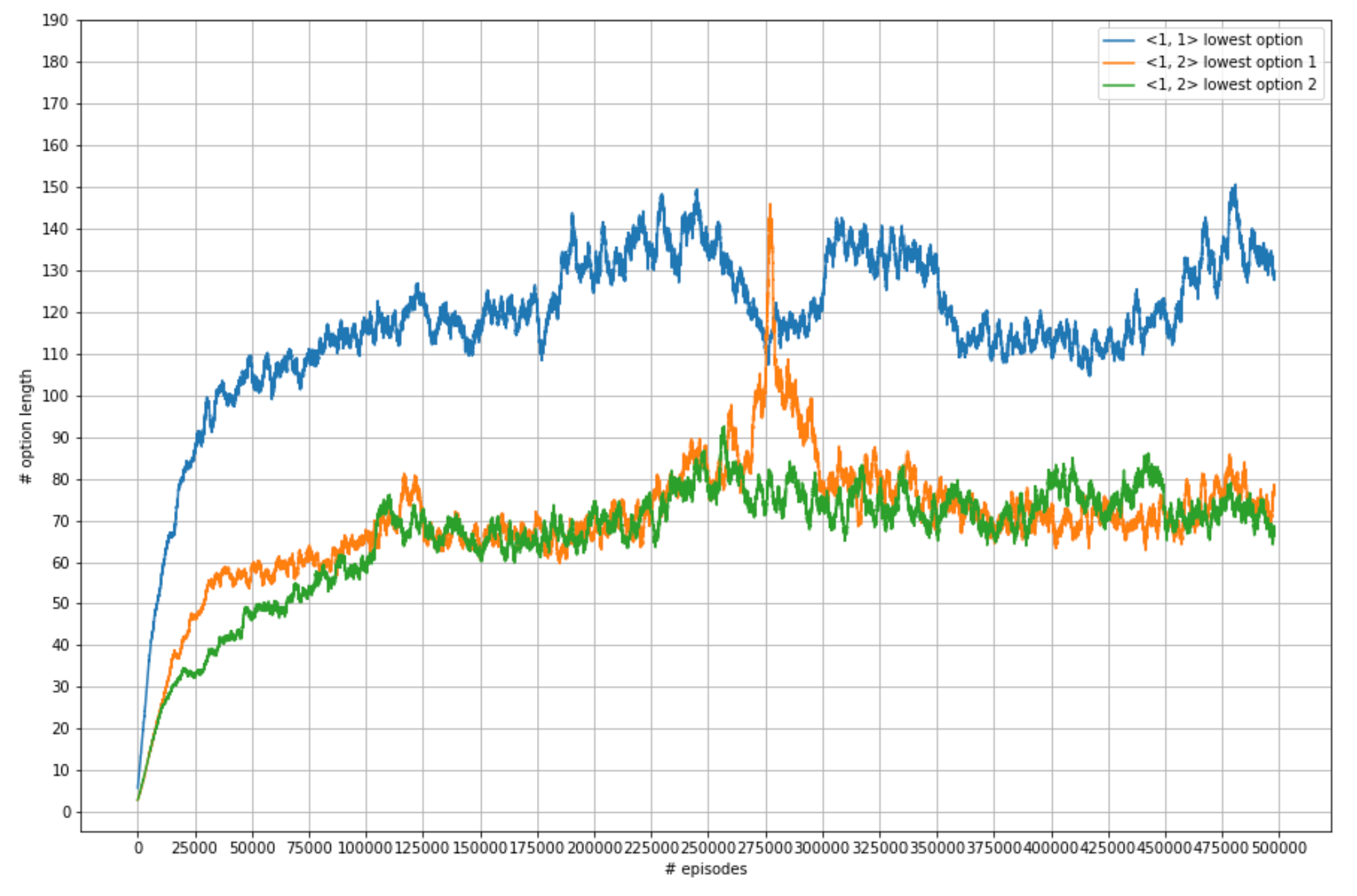}
    \includegraphics[width = \columnwidth]{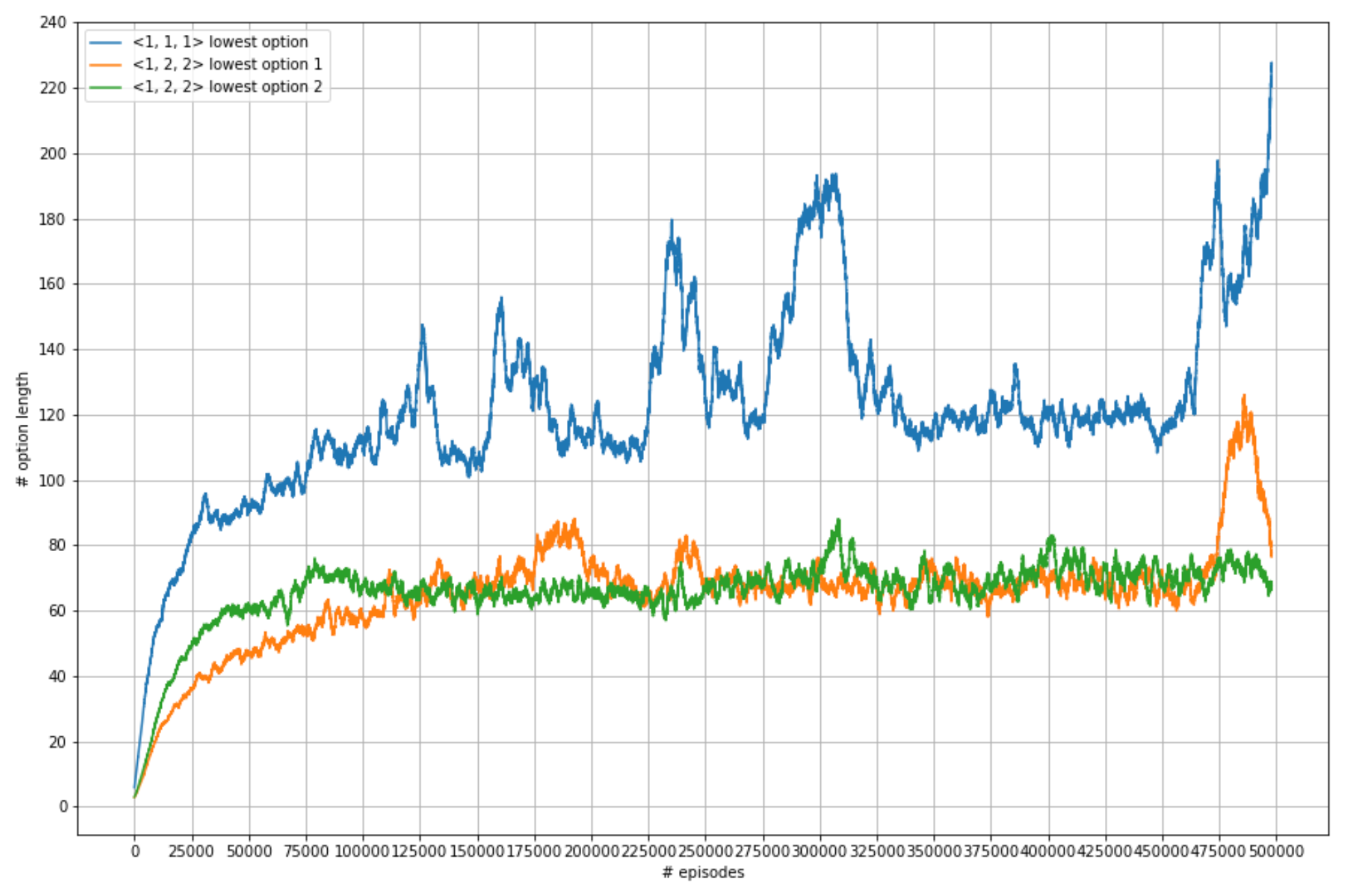}
    \includegraphics[width = \columnwidth]{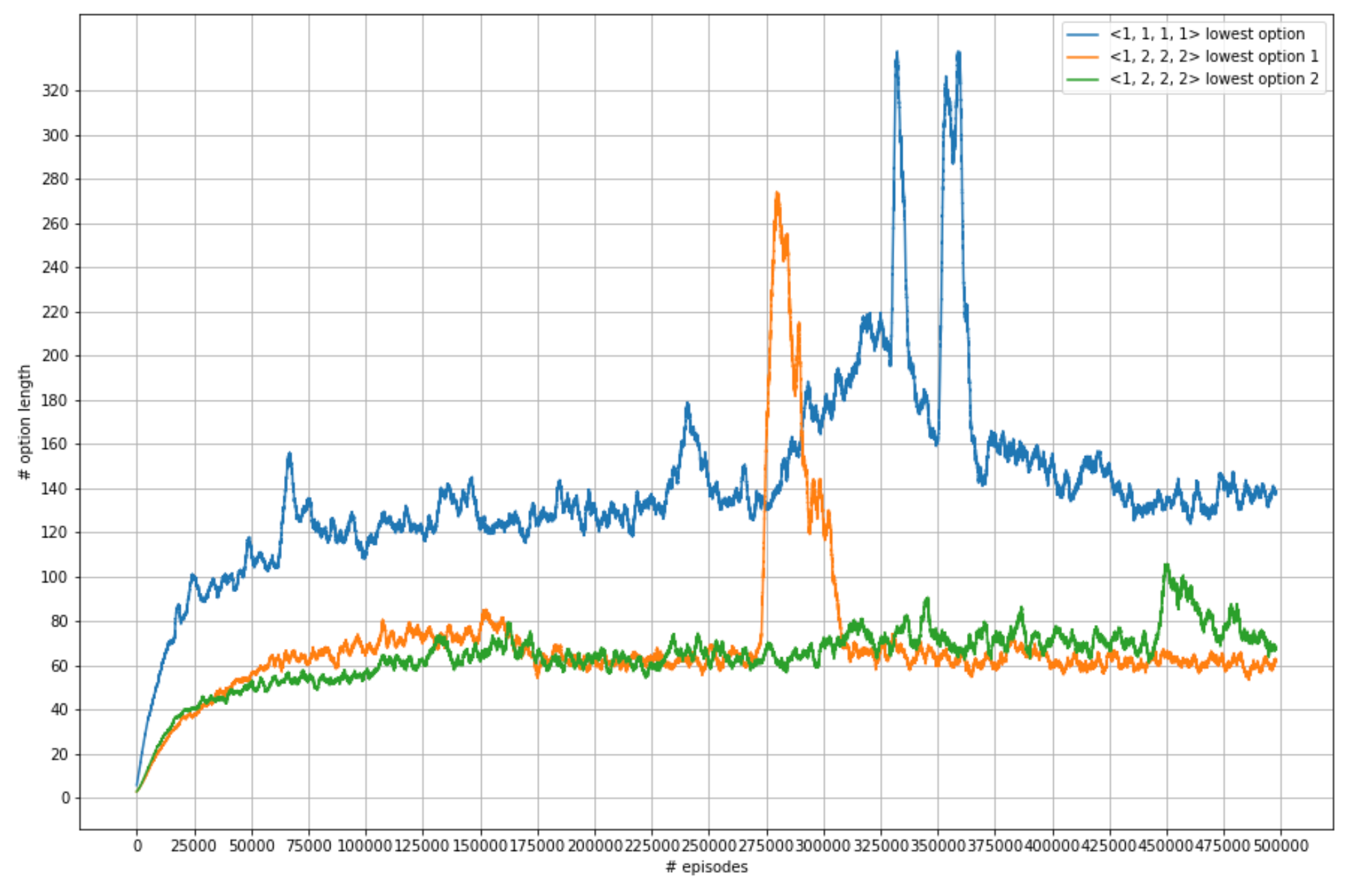}
    \includegraphics[width = \columnwidth]{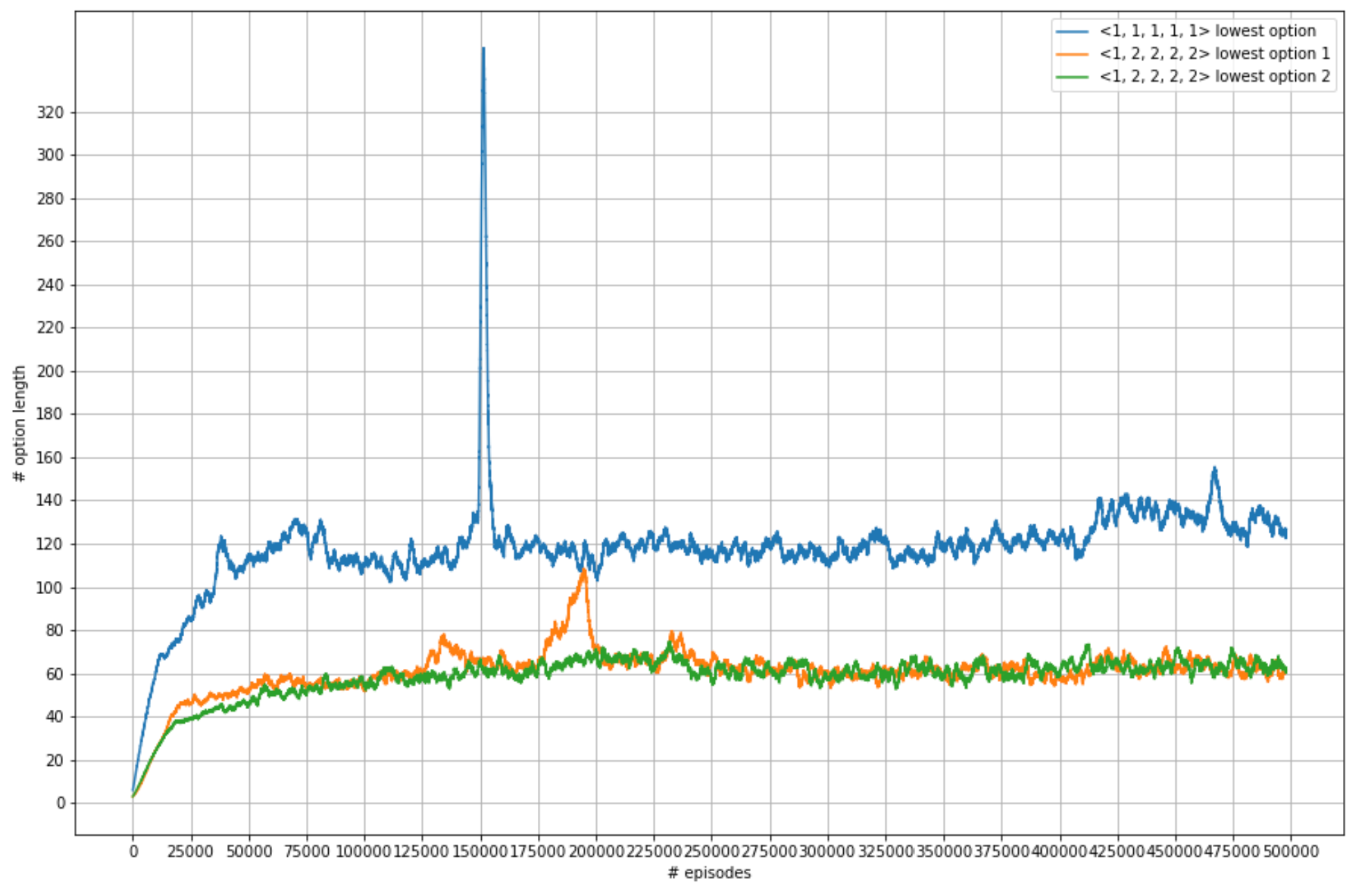}
    \includegraphics[width = \columnwidth]{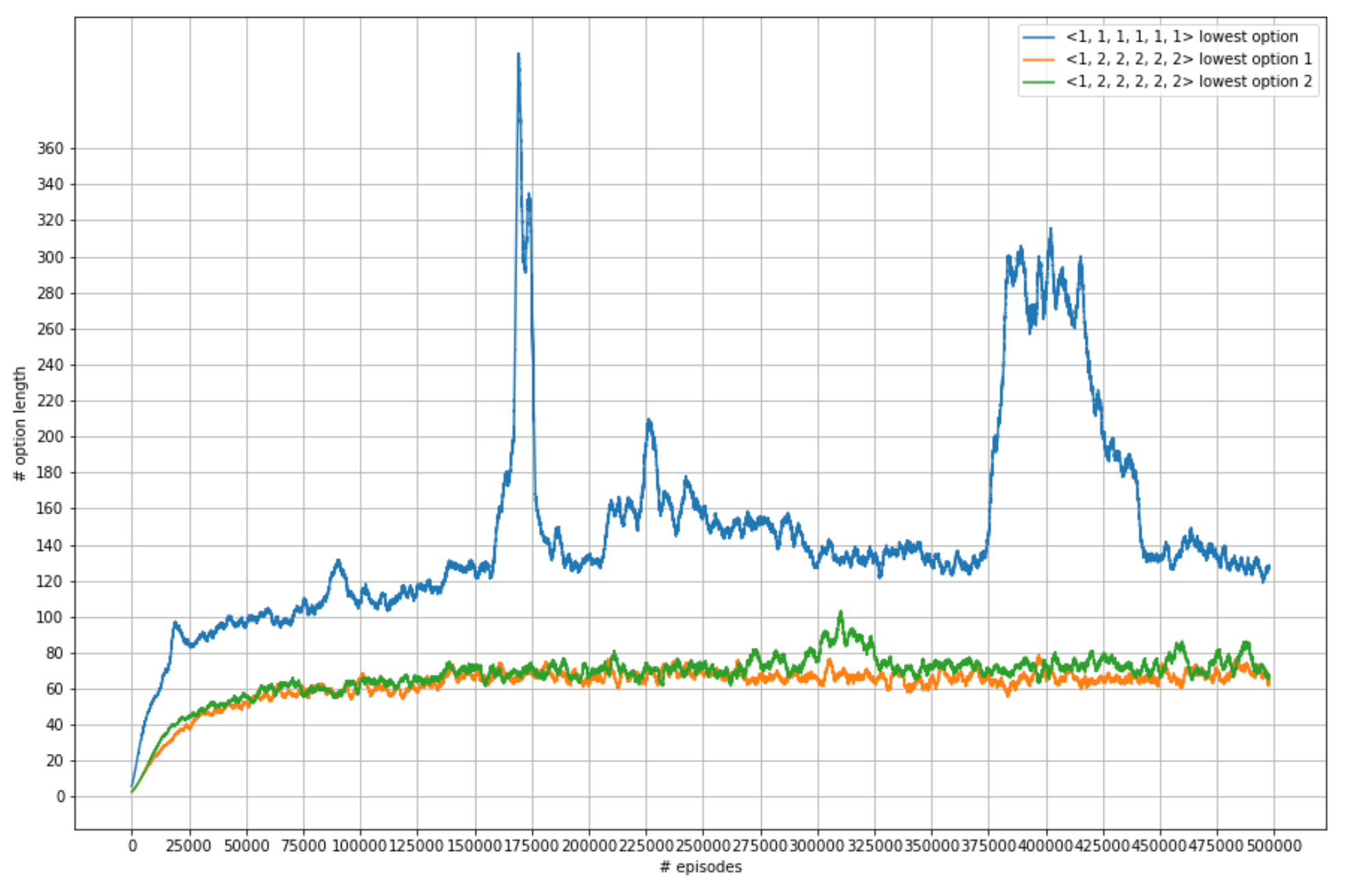}
    \includegraphics[width = \columnwidth]{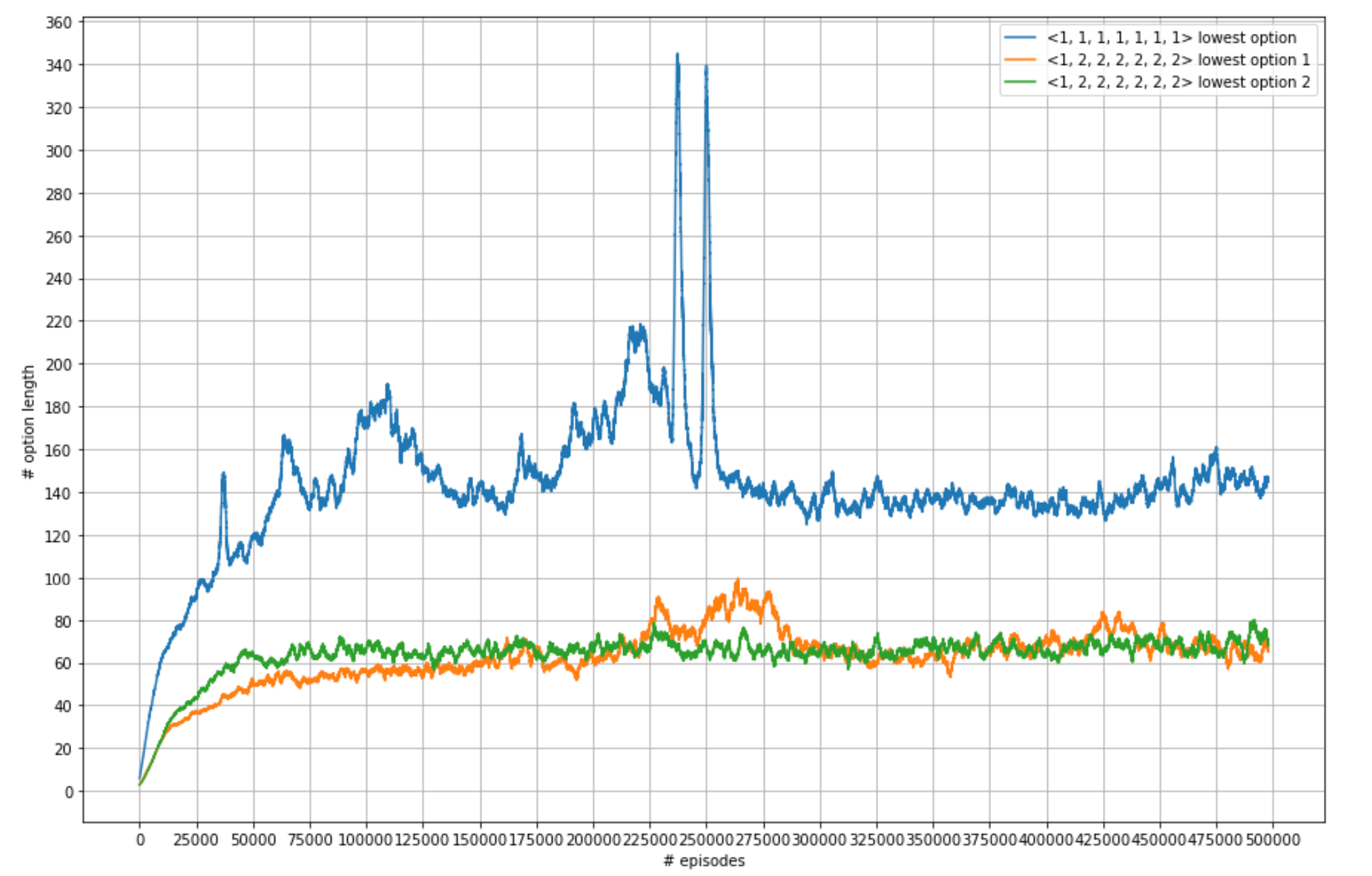}
    \caption{Lowest option(s) length comparison of $\left<1,\ldots,1\right>$ and $\left<1,2,\ldots,2\right>$ FON models with the same number of levels of hierarchy (two to seven). General stability increases as we switch from the former to the latter model. The lowest options of the latter model are enumerated by $1$ and $2$. We observe how both of the lowest options are almost equally used, with a substantial length, confirming the previous results in \cref{12and122optionlength} for even deeper models. Note the correlation between option length spikes and performance crashes in \cref{levlongrunstability}.}
    \label{levlongrunoptlength}
\end{figure*}

\end{document}